\documentclass[10pt]{article}  
\usepackage{color}
\usepackage{fullpage,graphicx,rotating,setspace} 
\usepackage[numbers]{natbib}
\usepackage{amsmath,amssymb,amscd,amsthm,amsfonts,mymath}   \graphicspath{{./}}
\begin{document}
\newtheorem{corr}{Corollary} 

\title{Estimating Probabilities in Recommendation Systems}
\author{Mingxuan Sun\footnote{Corresponding author: msun3@gatech.edu}  \hspace{0.2in} Guy Lebanon \\ {\small Georgia Institute of Technology} \\ {\small Atlanta GA 30332 USA} \and Paul Kidwell\\  {\small Lawrence Livermore National Lab}\\ {\small Livermore CA 94550 USA}}
\maketitle
\begin{abstract}
Recommendation systems are emerging as an important business application with significant economic impact. Currently popular systems include Amazon's book recommendations, Netflix's movie recommendations, and Pandora's music recommendations. In this paper we address the problem of estimating probabilities associated with recommendation system data using non-parametric kernel smoothing. In our estimation we interpret missing items as randomly censored observations and obtain efficient computation schemes using  combinatorial properties of generating functions. We demonstrate our approach with several case studies involving real world movie recommendation data. The results are comparable with state-of-the-art techniques while also providing probabilistic preference estimates outside the scope of traditional recommender systems.
\end{abstract}

\section{Introduction} \label{sec:intro} 

Recommendation systems are emerging as an important business application with significant economic impact. The data in such systems are collections of incomplete tied preferences across $n$ items associated with $m$ different users. Given an incomplete tied preference associated with an additional $m+1$ user, the system recommends unobserved items to that user based on the preference relations of the $m+1$ users. Currently deployed recommendation systems include book recommendations at amazon.com, movie recommendations at netflix.com, and music recommendations at pandora.com. Constructing accurate recommendation systems (that recommend to users items that are truly preferred over other items) is important for assisting users as well as increasing business profitability. It is an important unsolved goal in machine learning and data mining.

In most cases of practical interest the number of items $n$ indexed by the system (items may be books, movies, songs, etc.) is relatively high in the $10^3-10^4$ range. Perhaps due the size of $n$, it is almost always the case that each user observes only a small subset of the items, typically in the range 10-100. As a result the preference relations expressed by the users are over a small subset of the $n$ items. 

Formally, we have $m$ users providing incomplete tied preference relations on $n$ items 
\begin{align} \nonumber
S_1:\quad A_{1,1} \prec &A_{1,2}\prec \cdots \prec  A_{1,k(1)}  \\  \nonumber
S_2:\quad A_{2,1} \prec &A_{2,2} \prec \cdots \prec  A_{2,k(2)} \\
&\vdots \label{eq:data} \\  
S_m:\quad A_{m,1} \prec &A_{m,2} \prec \cdots \prec  A_{m,k(m)}  
\nonumber
\end{align}
where $A_{i,j}\subset \{1,\ldots,n\}$ are sets of items (wlog we identify items with integers $1,\ldots,n$) defined by the following interpretation: user $i$ prefers all items in $A_{i,j}$ to all items in $A_{i,j+1}$. The notation $k(i)$ above is the number of such sets provided by user $i$.  The data \eqref{eq:data} is incomplete since not all items are necessarily observed by each user i.e., $\bigcup_{j=1}^{k(i)} A_{i,j} \subsetneq \{1,\ldots,n\}$ and may contain ties since some items are left uncompared, i.e., $|A_{i,j}|>1$. Recommendation systems recommend items to a new user, denoted as $m+1$, based on their preference
\begin{align} \label{eq:activeUser}
S_{m+1}:A_{m+1,1} \prec &A_{m+1,2} \prec \cdots \prec  A_{m+1,k(m+1)}  
\end{align}
and its relation to the preferences of the $m$ users \eqref{eq:data}.

As an illustrative example, assuming $n=9, m=3$, the data 
\begin{align*}
&S_1:\qquad 1,8,9 \prec 4 \prec 2,3,7\\
&S_2:\qquad 4\prec 2,3 \prec 8\\
&S_3:\qquad 4,8\prec 2,6,9
\end{align*}
corresponds to $A_{1,1}=\{1,8,9\}$, $A_{1,2}=\{4\}$, $A_{1,3}=\{2,3,7\}$, $A_{2,1}=\{4\}$, $A_{2,2}=\{2,3\}$, $A_{2,3}=\{8\}$, $A_{3,1}=\{4,8\}$, $A_{3,2}=\{2,6,9\}$, and $k(1)=k(2)=3, k(3)=2$. From the data we may guess that item 4 is relatively popular across the board while some users like item 8 (users 1, 3) and some hate it (user 2). Given a new $m+1$ user issuing the preference $1 \prec 2,3,7$ we might observe a similar pattern of preference or taste as user 1 and recommend to the user item 8. We may also recommend item 4 which has broad appeal resulting in the augmentation
\[ 1 \prec 2,3,7 \qquad \mapsto \qquad  1,4,8 \prec 2,3,7.\]

We note that in some cases the preference relations \eqref{eq:data} arise from users providing numeric scores to items. For example, if the users assign 1-5 stars to movies, the set $A_{i,j}$ contains all movies that user $i$ assigned $6-j$ stars to and $k(i)=5$ (assuming some movies were assigned to each of the 1, 2, 3, 4, 5 star levels). As pointed out by a wide variety of studies in economics and social sciences, such numeric scores are inconsistent among different users. We therefore proceed to interpret such data as ordinal rather than numeric.

A substantial body of literature in computer science has addressed the problem of constructing recommendation systems. We have attempted to outline the most important and successful approaches in the related work section towards the end of this paper. However, none of these previous approaches are fully satisfactory from a statistical perspective: there are no reasonable probability models assumed to generate the data and no clear meaningful statistical estimation procedures. We substantiate this argument more fully in the related work section.

In this paper we describe a non-parametric statistical technique for estimating probabilities on preferences based on the data \eqref{eq:data}. This technique may be used in recommendation systems in different ways. Its principal usage may be to provide a statistically meaningful estimation framework for issuing recommendations (in conjunction with decision theory). However, it also leads to other important applications including mining association rules, exploratory data analysis, and clustering items and users. Two key observations that we make are: (i) incomplete tied preference data may be interpreted as randomly censored permutation data, and (ii) using generating functions we are able to provide a computationally efficient scheme for computing the estimator in the case of triangular smoothing. 

We proceed in the next sections to describe notations and our  assumptions and estimation procedure, and follow with case studies demonstrating our approach on real world recommendation systems data.

\section{Definitions and Estimation Framework}

We describe the following notations and conventions for permutations, which are taken from \cite{Diaconis:88} where more detail may be found. We denote a permutation by listing the items from most preferred to least separated by a $\prec$ or $|$ symbol: $\pi^{-1}(1)\prec \pi^{-1}(2)\prec \cdots\prec \pi^{-1}(n)$, e.g. $\pi(1)=2,\pi(2)=3,\pi(3)=1$ is $3\prec 1\prec 2$. Ranking with ties occur when judges do not provide enough information to construct a total order. In particular, we define tied rankings as a partition of $\{1,\ldots,n\}$ to $k< n$ disjoint subsets $A_1,\ldots,A_k\subset \{1,\ldots,n\}$ such that all items in $A_i$ are preferred to all items in $A_{i+1}$ but no information is provided concerning the relative preference of the items among the sets $A_i$. We denote such rankings by separating the items in $A_i$ and $A_{i+1}$ with a $\prec$ or $|$ notation. For example, the tied ranking  $A_1=\{3\}, A_2=\{2\}, A_3=\{1,4\}$ (items 1 and 4 are tied for last place) is denoted as $3\prec 2\prec 1,4$ or $3|2|1,4$.

Ranking with missing items occur when judges omit certain items from their preference information altogether. For example assuming a set of items $\{1,\ldots,4\}$, a judge may report a preference $3\prec 2\prec 4$, omitting altogether item 1 which the judge did not observe or experience. This case is very common in situations involving a large number of items $n$. In this case judges typically provide preference only for the $l\ll n$ items that they observed or experienced. For example, in movie recommendation systems we may have $n\sim 10^3$ and $l\sim 10^1$.

Rankings can be full (permutations), with ties, with missing items, or with both ties and missing items. In either case we denote the rankings using the $\prec$ or $|$ notation or using the disjoint sets $A_1,\ldots,A_k$ notation. We also represent tied and incomplete rankings by the set of permutations that are consistent with it. For example,
\begin{align*}
  3\prec 2\prec 1,4 &= \{3\prec 2\prec 1\prec 4\}\cup\{ 3\prec   2\prec  4\prec  1\} \\
  3\prec 2\prec 4 &=\{1\prec 3\prec 2\prec 4\}\cup\{3\prec 1\prec   2\prec 4\} \cup\{3\prec 2\prec 1\prec 4\}\cup\{3\prec 2\prec 4\prec   1\}
\end{align*}
are sets of two and four permutations corresponding to tied and incomplete rankings, respectively.

It is hard to directly posit a coherent probabilistic model on incomplete tied data such as \eqref{eq:data}. Different preferences relations are not unrelated to each other: they may subsume one another (for example $1\prec 2 \prec 3$ and $1 \prec 3$), represent disjoint events (for example $1\prec 3$ and $3\prec 1$), or interact in more complex ways (for example $1 \prec 2 \prec 3$ and $1 \prec 4 \prec 3$). A valid probabilistic framework needs to respect the constraints resulting from the axioms of probability, e.g., $p(1 \prec 2 \prec 3) \leq p(1 \prec 3)$. 

Our approach is to consider the incomplete tied preferences as censored permutations. That is, we assume a distribution $p(\pi)$ over permutations $\pi\in\S_n$ ($\S_n$ is the symmetric group of permutations of order $n$) that describes the complete without-ties preferences in the population. The data available to the recommender system \eqref{eq:data} is sampled by drawing $m$ iid permutations from $p$: $\pi_1,\ldots,\pi_m\iid p$, followed by censoring to result in the observed preferences $S_1,\ldots,S_m$ 
\begin{align}
\pi_i &\sim p(\pi),\quad S_i\sim p(S|\pi_i), \quad i=1,\ldots,m+1 \\
p(\pi|S) &= \frac{I(\pi\in S)p(\pi)}{\sum_{\sigma\in S} p(\sigma)} \label{eq:condProb}
\\
p(S|\pi) &= p(\pi|S)p(S)/p(\pi) =  \frac{I(\pi\in S) p(\pi) p(S)}{p(\pi)\sum_{\sigma\in S} p(\sigma)} =
\frac{I(\pi\in S) p(S)}{\sum_{\sigma\in S} p(\sigma)}
\end{align}
where $p(S)$ is the probability of observing the censoring $S$ (specifically, it is not equal to $\sum_{\sigma\in S} p(\sigma)$). 

Although many approaches for estimating $p$ given $S_1,\ldots,S_m$ are possible, experimental evidence point to the fact that in recommendation systems with high $n$, the distribution $p$ does not follow a simple parametric form such as the Mallows, Bradley-Terry, or Thurstone models  \cite{Marden1996} (see Figure~\ref{fig:heatMaps} for a demonstration how parametric assumptions break down with increasing $n$). Instead, the distribution $p$ tends to be diffuse and multimodal  with different probability mass regions corresponding to different types of judges (for example in movie preferences probability modes may correspond to genre as fans of drama, action, comedy, etc. having similar preferences). 

\begin{figure}
\centering
\begin{tabular}{lll}
\includegraphics[scale=.24,angle=0]{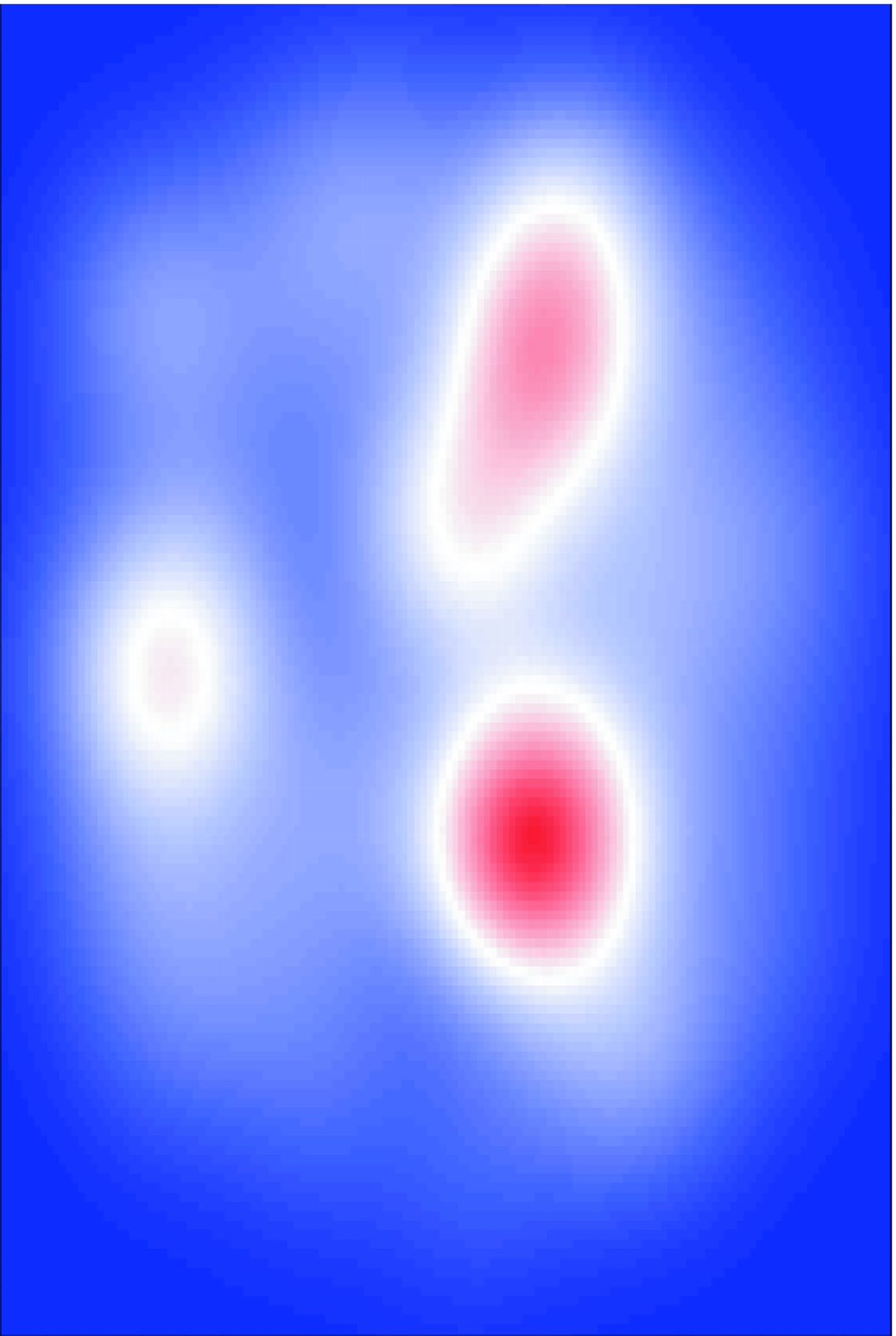}&\hspace{-0.1in}
\includegraphics[scale=.235,angle=0]{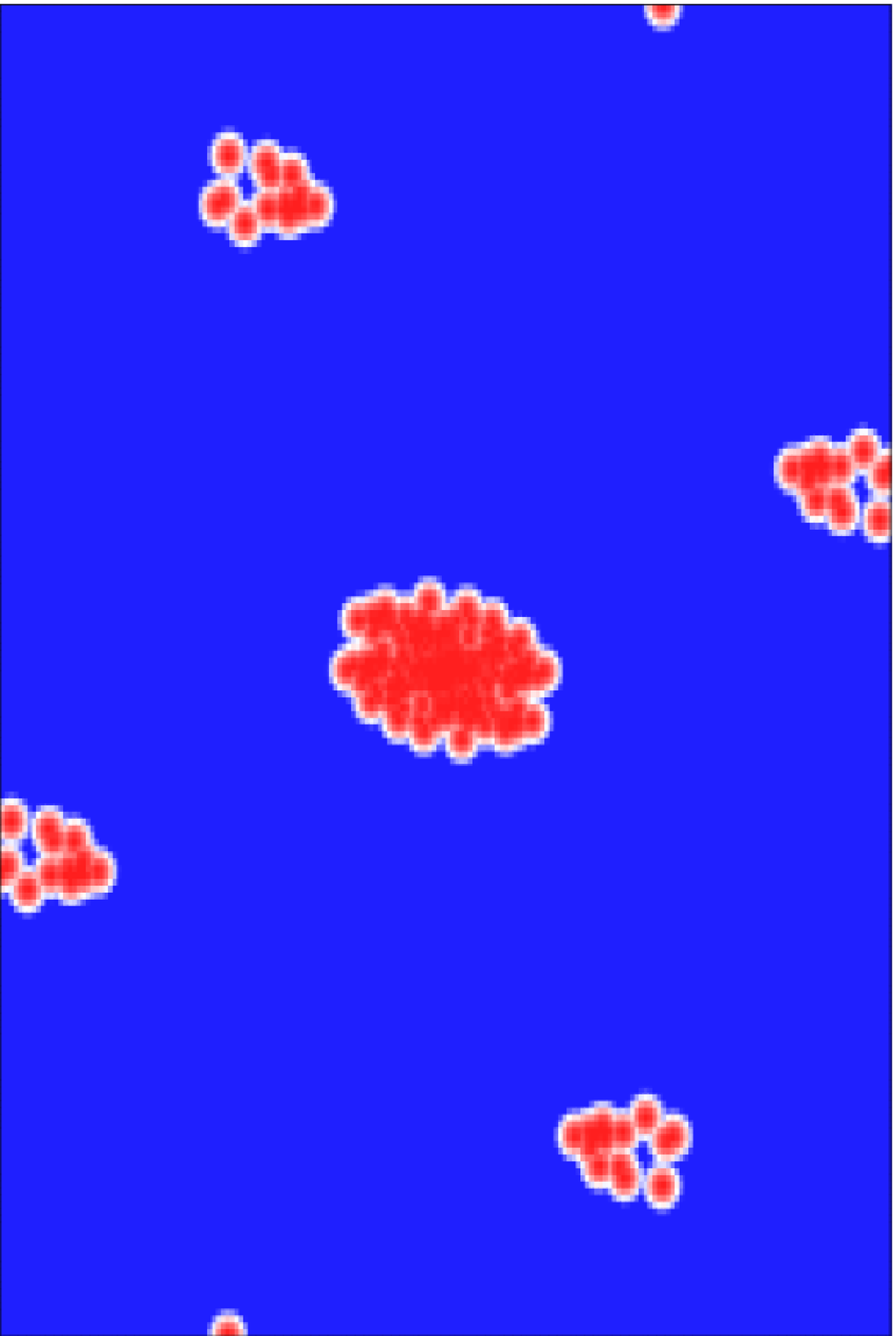}&\hspace{-0.1in}
\includegraphics[scale=.24,angle=0]{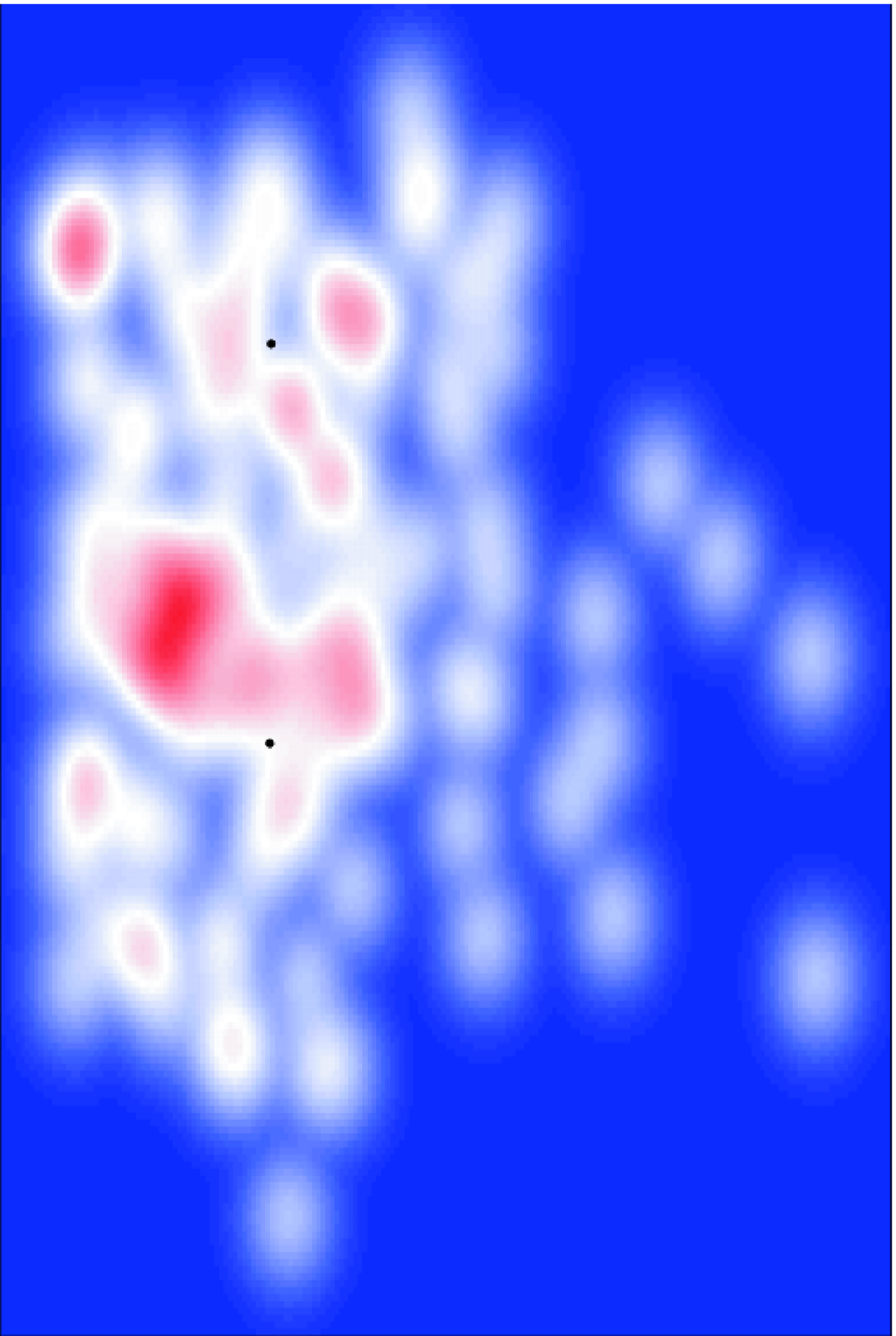}
\end{tabular}
\caption{Heat map visualization of the density of ranked data using multidimensional scaling with expected Kendall's Tau distance. The datasets are APA voting (left, $n=5$), Jester (middle, $n=100$), and EachMovie (right, $n=1628$) datasets. None of these cases show a simple parametric form, and the complexity of the density increases with the number of items $n$. This motivates the use of non-parametric estimators for modeling preferences over a large number of items.}
\label{fig:heatMaps}
\end{figure}

We therefore propose to estimate the underlying distribution $p$
on permutations using non-parametric kernel smoothing. The standard kernel smoothing formula applies to the permutation setting as
\[ \hat p(\pi) = \frac{1}{m}\sum_{i=1}^m K_h(T(\pi,\pi_i)) \]
where $\pi_1,\ldots,\pi_m\iid p$, $T$ a distance on permutations such as Kendall's distance and $K_h(r)=h^{-1} K(r/h)$  a normalized unimodal function. In the case at hand, however, the observed preferences $\pi_i$ as well as  $\pi$ are replaced with permutations sets $S_1,\ldots,S_m, R$ representing incomplete tied preferences 
\begin{align} 
\hat p(R) = \sum_{\pi\in R} \hat p(\pi) = \frac{1}{m} \sum_{i=1}^m \sum_{\pi\in R}    \sum_{\sigma\in S_i} q(\sigma|S_i) K_h(T(\pi,\sigma)) \label{eq:npEstimator}
\end{align}
where $q(\sigma|S_i)$ serves as a surrogate for the unknown  $p(\sigma|S_i)\propto I(\sigma\in S_i)p(\sigma)$ (see \eqref{eq:condProb}). Selecting $q(\sigma|S_i)=p(\sigma|S_i)$ would lead to consistent estimation of $p(R)$ in the limit $h\to 0$, $m\to\infty$ assuming positive  $p(\pi), p(S)$. Such a selection, however, is generally impossible since $p(\pi)$ and therefore $p(\sigma|S_i)$ are unknown. 

In general the specific choice of the surrogate $q(\sigma|S)$ is important as it may influence the estimated probabilities. Furthermore, it may cause underestimation or overestimation of $\hat p(R)$ in the limit of large data. An exception occurs when the sets $S_1,\ldots,S_m$ are either subsets of $R$ or disjoint from $R$. In this case $\lim_{h\to 0} K_h(\pi,\sigma) = I(\pi=\sigma)$ resulting in the following limit (with probability 1 by the strong law of large numbers)
\begin{align}
\lim_{m\to \infty}\,\, \lim_{h\to 0} \,\, \hat p(R) &= 
\lim_{m\to \infty}\,\,  \frac{1}{m} \sum_{i=1}^m I(S_i\subset R) \sum_{\sigma\in S_i}   q(\sigma|S_i)\\
&= \lim_{m\to \infty}\,\, \frac{1}{m} \sum_{i=1}^m I(S_i\subset R)= 
\lim_{m\to \infty}\,\, \frac{1}{m} \sum_{i=1}^m I(\pi_i\in R)= p(R).
\end{align}
Thus, if we our data is comprised of preferences $S_i$ that are either disjoint or a subset of $R$ we have consistency \emph{regardless} of the choice of the surrogate $q$. Such a situation is more realistic when the preference $R$ involves a small number of items and the preferences $S_i$, $i=1,\ldots,m$ involve a larger number of items. This is often the case for recommendation systems where individuals report preferences over 10-100 items and we are mostly interested in estimating probabilities of preferences over fewer items such as  $i\prec j,k$ or $i\prec j,k\prec l$ (see experiment section). 

The main difficulty with the estimator above is the computation of 
$\sum_{\pi\in R}\sum_{\sigma\in S_i} q(\sigma|S_i) K_h(T(\pi,\sigma))$. In the case of high $n$ and only a few observed items $k$ the sets $S_i,R$ grow factorially as $(n-k)!$ making a naive computation of \eqref{eq:npEstimator}  intractable for all but the smallest $n$. In the next section we explore efficient computations of these sums for a triangular kernel $K_h$ and a uniform $q(\pi|S)$.

\section{Computationally Efficient Kernel Smoothing}

In previous work \cite{Lebanon2008} the estimator \eqref{eq:npEstimator} is proposed for tied (but complete) rankings. That work derives closed form expressions and efficient computation for \eqref{eq:npEstimator} assuming a Mallows kernel \cite{Mallows1957}
\begin{align} \label{eq:MallowsKernel}
K_h(T(\pi,\sigma)) &=  \exp\left( -  \frac{T(\pi,\sigma)}{h}\right) \prod_{j=1}^n \frac{1-e^{-1/h}}{1-e^{-j/h}}
\end{align}
where $T$ is Kendall's Tau distance on permutations (below  $I(x)=1$ for $x>0$ and 0 otherwise)
\begin{align}
T(\pi,\sigma) &= \sum_{i=1}^{n-1}\sum_{l>i} I(\pi\sigma^{-1}(i)-\pi\sigma^{-1}(l)).
\end{align}
Unfortunately these simplifications do not carry over to the case of incomplete rankings where the sets of consistent permutations $S_1,\ldots,S_m$ are not cosets of the symmetric group. As a result the problem of probability estimation in recommendation systems where $n$ is high and many items are missing is particularly challenging. However, as we show below replacing the Mallows kernel \eqref{eq:MallowsKernel} with a triangular kernel leads to efficient computation in some cases. Specifically, the triangular kernel on permutation is 
\begin{align} \label{eq:triangularKernelRankedData}   K_{h}(T(\pi,\sigma)) &= (1-h^{-1}T(\pi,\sigma)) \, I(h-T(\pi,\sigma)) \, /\,  C
\end{align}
where the bandwidth parameter $h$ represent both the support (the kernel is 0 for all larger distances) and the inverse slope of the triangle. As we show below the normalization term $C$ is a function of $h$ and may be efficiently computed using generating functions. Figure~\ref{fig:kernelPlot} (right panel) displays the linear decay of \eqref{eq:triangularKernelRankedData} for the simple case of permutations over $n=3$ items.

\begin{figure}\hspace{-0.33in}
  \begin{minipage}{0.33\textwidth}
    \begin{tabular}{c}
      \includegraphics[width=\textwidth]{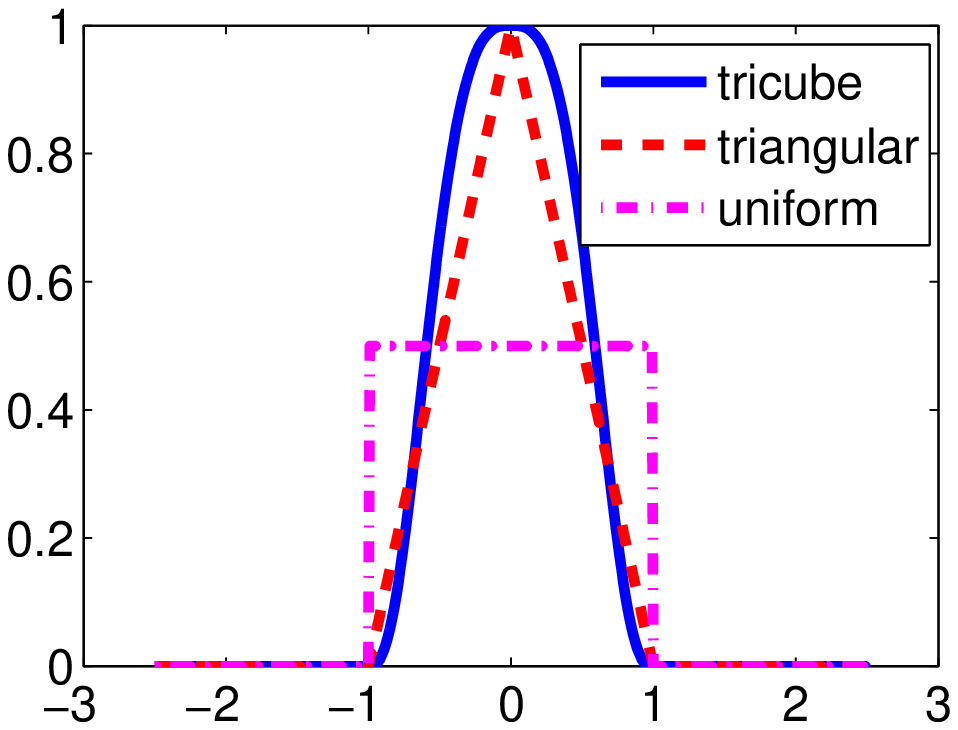}
    \end{tabular}
  \end{minipage}
  \hspace{-0.23in}
  \begin{minipage}{0.33\textwidth}
    \begin{tabular}{c}
      \includegraphics[width=\textwidth]{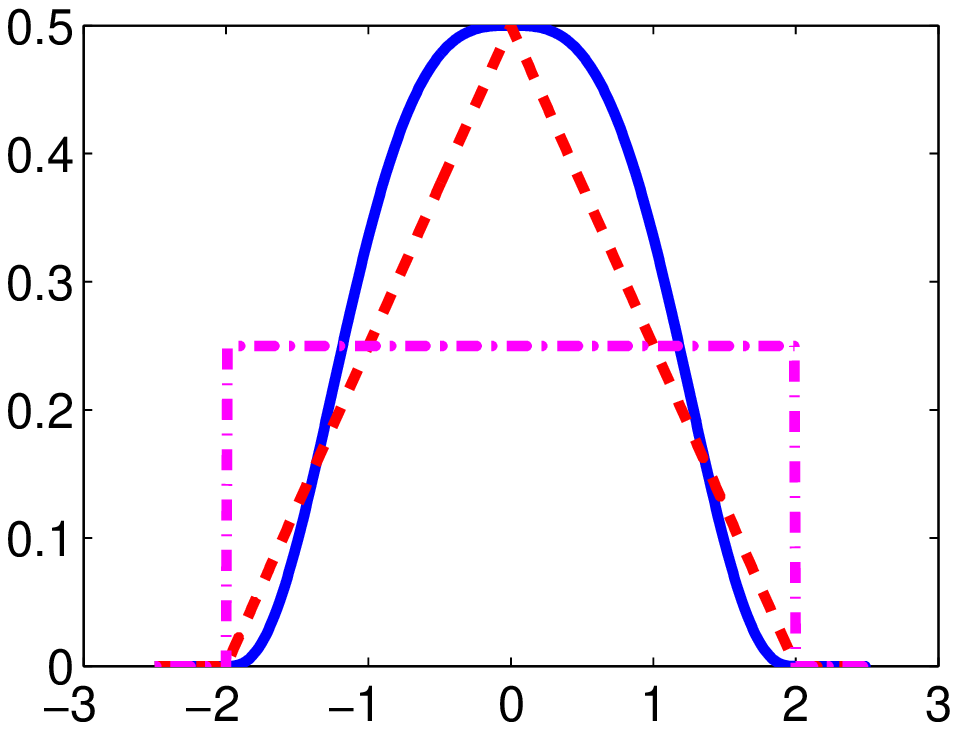}
    \end{tabular}
  \end{minipage}
  \hspace{-0.2in}
  \begin{minipage}{0.33\textwidth}
\begin{tabular}{c}
{\scriptsize \begin{tabular}{|l|ll|}
\hline
 & $\!\!K_{3}(\cdot,1\prec 2\prec 3)$  & $\!\!\!\!\!\!K_{5}(\cdot,1\prec 2\prec 3)$  \\
\hline \hline
$1\prec 2 \prec 3$ & 0.50& 0.33 \\ 
$1\prec 3 \prec 2$ & 0.25& 0.22 \\ 
$2\prec 1 \prec 3$ & 0.25& 0.22 \\ 
$3\prec 1 \prec 2$ & 0 &0.11  \\ 
$2\prec 3 \prec 1$ & 0 &0.11  \\ 
$3\prec 2 \prec 1$ & 0 &0  \\ 
\hline
\end{tabular}}
\end{tabular}
\end{minipage}
\caption{Tricube, triangular, and uniform kernels on $\mathbb{R}$ with   bandwidth $h=1$ (left) and $h=2$ (middle). Right: triangular kernel on   permutations ($n=3$).}\label{fig:kernelPlot}
\end{figure}

\subsubsection*{Combinatorial Generating Function}
Generating functions, a tool from enumerative combinatorics, allow efficient computation of \eqref{eq:npEstimator} by concisely expressing the distribution of distances between permutations. Kendall's tau $T(\pi,\sigma)$ is the total number of discordant pairs or inversions between $\pi,\sigma$ \cite{Stanley2000} and thus its computation becomes a combinatorial counting problem. We associate the following generating function with the symmetric group of order $n$ permutations 
\begin{align}
G_n(z)=\prod_{j=1}^{n-1} \sum_{k=0}^{j}z^k.
\end{align}
As shown for example in \cite{Stanley2000} the coefficient of $z^k$ of $G_n(z)$, which we denote as $[z^k]G_n(z)$, corresponds to the number of permutations $\sigma$ for which $T(\sigma,\pi')=k$. For example, the distribution of Kendall's tau $T(\cdot,\pi')$ over all permutations of 3 items is described by
$G_3(z)=(1+z)(1+z+z^2)=1z^0+2z^1+2z^2+1z^3$
i.e., there is one permutation $\sigma$ with $T(\sigma,\pi')=0$, two permutations $\sigma$ with $T(\sigma,\pi')=1$, two with $T(\sigma,\pi')=2$ and one with $T(\sigma,\pi')=3$. Another important generating function is
\[ H_n(z)=\frac{G_n(z)}{1-z} = (1+z+z^2+z^3+\cdots)G_n(z)\] 
where $[z^k]H_n(z)$ represents the number of permutations $\sigma$ for which $T(\sigma,\pi')\leq k$.
\begin{prop}
The normalization term $C(h)$ is given by 
$C(h) = [z^h] H_n(z)-h^{-1}[z^{h-1}]\frac{G_n'(z)}{1-z}.$
\end{prop}
\begin{proof}
The proof factors the non-normalized triangular kernel $C K_h(\pi,\sigma)$ to $I(h-T(\pi,\sigma))$ and $h^{-1}T(\pi,\sigma)I(h-T(\pi,\sigma))$ and making the following observations. First we note that summing the first factor over all permutations may be counted by $[z^h]H_n(z)$. The second observation is that $[z^{k-1}]G_n'(z)$ is the number of permutations $\sigma$ for which $T(\sigma,\pi')=k$, multiplied by $k$. Since we want to sum over that quantity for all permutations whose distance is less than $h$ we extract the $h-1$ coefficient of the generating function 
$G_n'(z)\sum_{k\geq 0} z^k=G_n'(z)/(1-z)$. We thus have
\begin{align*}
  C=\sum_{\sigma: T(\pi',\sigma)\leq h}1-h^{-1}\sum_{\sigma: T(\pi',\sigma)\leq h} T(\pi',\sigma) =  [z^{h}]H_n(z)-h^{-1}[z^{h-1}]\frac{G_n'(z)}{1-z}.
\end{align*}  
\end{proof}
\begin{prop}
The complexity of computing $C(h)$ is $O(n^4)$.
\end{prop}
\begin{proof}
We describe a dynamic programming algorithm to compute the coefficients of $G_n$ by recursively computing the coefficients of $G_k$ from the coefficients of $G_{k-1}$, $k=1,\ldots,n$. The generating function $G_k(z)$ has $k(k+1)/2$ non-zero coefficients and computing each of them (using the coefficients of $G_{k-1}$) takes $O(k)$. We thus have $O(k^3)$ to compute $G_k$ from $G_{k-1}$ which implies $O(n^4)$ to compute $G_k$, $n=1,\ldots,n$. We conclude the proof by noting that once the coefficients of $G_n$ are computed the coefficients of $H_n(z)$ and $G_n(z)/(1-z)$ are computable in $O(n^2)$ as these are simply cumulative weighted sums of the coefficients of $G_n$.
\end{proof}
Note that computing $C(h)$ for one or many $h$ values may be done offline prior to the arrival of the rankings and the need to compute the estimated probabilities.

Denoting by $k$ the number of items ranked in either $S$ or $R$ or both, the computation of $\hat p(\pi)$ in \eqref{eq:npEstimator} requires $O(k^2)$ online and $O(n^4)$ offline complexity if either non-zero smoothing is performed over the entire data i.e., $\max_{\pi\in R}\max_{i=1}^n\max_{\sigma\in S_i} T(\sigma,\pi)<h$ or alternatively, we use the modified triangular kernel  $K_h^*(\pi,\sigma)\propto (1-h^{-1})T(\pi,\sigma)$ which is allowed to take negative values for the most distant permutations (normalization still applies though). 

\begin{prop} For two sets of permutations $S,R$ corresponding to tied-incomplete rankings
\begin{align} \label{eq:expectedTau}
\frac{1}{|S||R|}\sum_{\pi\in S}\sum_{\sigma\in R} T(\pi,\sigma) &= \frac{n(n-1)}{4}-\frac{1}{2}\sum_{i=1}^{n-1}\sum_{j=i+1}^n   (1-2p_{ij}(S))(1-2p_{ij}(R))\\
p_{ij}(U) &= \begin{cases}
    I(\tau_U(j)-\tau_U(i)) & \mbox{$i$ and $j$ are ranked in $U$ with  $\tau_U(i)\neq\tau_U(j)$}\\
    1-\frac{\tau_U(i)+\frac{\phi_U(i)-1}{2}}{k+1} & \mbox{only       $i$ is ranked in $U$}\\
    \frac{\tau_U(j)+\frac{\phi_U(j)-1}{2}}{k+1} & \mbox{only       $j$ is ranked in $U$}\\
    1/2 & \mbox{otherwise}\end{cases}. \nonumber
\end{align}
with $\tau_U(i)=\min_{\pi\in U} \pi(i)$, and $\phi_U(i)$ being the number of items that are tied to $i$ in $U$.
\end{prop}
\begin{proof}
We note that \eqref{eq:expectedTau} is an expectation with respect to the uniform measure. We thus start by computing the probability $p_{ij}(U)$ that $i$ is preferred to $j$ for $U=S$ and $U=R$ under the uniform measure. Five scenarios exist for each of $p_{ij}(U)$ corresponding to whether each of $i$ and $j$ are ranked by $S,R$. Starting with the case that $i$ is not ranked and $j$ is ranked, we note that $i$ is equally likely to be preferred to any item or to be preferred to. Given the uniform distribution over compatible rankings item $j$ is equally likely to appear in positions $\tau_U(j),\ldots,\tau_U(j)+\phi_U(j)-1$. Thus
\begin{align}
  p_{ij}&=\frac{1}{\phi_U(j)}\frac{\tau_U(j)}{k+1}+\cdots+\frac{1}{\phi_U(j)}\frac{\tau_U(j)+\phi_U(j)-1}{k+1}=\frac{\tau_U(j)+\frac{\phi_U(j)-1}{2}}{k+1}
\end{align}
Similarly, if $j$ is unknown and $i$ is known then $p_{ij}+p_{ji}=1$. If both $i$ and $j$ are unknown either ordering must be equally likely given the uniform distribution making $p_{ij}=1/2$. Finally, if  both $i$ and $j$ are known $p_{ij}=1,1/2,0$ depending on their preference. Given $p_{ij}$, linearity of expectation, and the independence between rankings, the change in the expected number of inversions relative to the uniform expectation $n(n-1)/4$ can be found by considering each pair separately,
\begin{eqnarray*}
 \mbox{E}T(i,j)&=&\frac{1}{2}P\left(\mbox{$i$ and $j$       disagree}\right)-\frac{1}{2}P\left(\mbox{$i$ and $j$ agree}\right)\\
  &=&\frac{1}{2}(p_{ij}(\sigma)(1-p_{ij}(\pi))+(1-p_{ij}(\sigma))p_{ij}(\pi)) -p_{ij}(\sigma)p_{ij}(\pi)-(1-p_{ij}(\sigma))(1-p_{ij}(\pi)))\\
  &=&\frac{-1}{2}\left(1-2p_{ij}(\sigma)\right)\left(1-2p_{ij}({\pi})\right).
\end{eqnarray*}
Summing the $n(n-1)/2$ components yields the desired quantity.
\end{proof}

\begin{corr} Denoting the number of items ranked by either $S$ or $R$ or both as $k$, and assuming either $h>\max_{\pi\in R}\max_{i=1}^n\max_{\sigma\in S_i} T(\sigma,\pi)$ or that the modified triangular kernel $K_h^*(\pi,\sigma)\propto (1-h^{-1})T(\pi,\sigma)$ is used, the complexity of computing $\hat p(R)$ in  \eqref{eq:npEstimator} (assuming uniform $q(\pi|S_i)$) is $O(k^2)$ online and $O(n^4)$ offline.
\end{corr} \label{corr:expDist}
\begin{proof}
The proof follows from noting that \eqref{eq:npEstimator} reduces to $O(n^4)$ offline computation of the normalization term and $O(k^2)$ online computation of the form \eqref{eq:expectedTau}.
\end{proof}

\section{Applications and Case Studies}

We divide our experimental study to three parts. In the first we examine the task of predicting probabilities. The remaining two parts use these probabilities for rank prediction and rule discovery.  

In our experiments we used three datasets. The Movielens dataset\footnote{http://www.grouplens.org} contains one million ratings from $6040$ users over $3952$ movies. 
The EachMovie dataset\footnote{ http://www.grouplens.org/node/76} contains $2.6$ million ratings from $74424$ users over $1648$ movies. The Netflix dataset \footnote{http://www.netflixprize.com/community} contains 100 million movie ratings from $480189$ users on $17770$.
In all of these datasets users typically rated only a small number of items. Histograms of the distribution of the number of votes per user, number of votes per item, and vote distribution appear in Figure~\ref{fig:hists}.

\begin{figure}
\begin{tabular}{ccc}
{\scriptsize Movielens}& {\scriptsize Netflix} &{\scriptsize EachMovie}    \\
\includegraphics[scale=0.34]{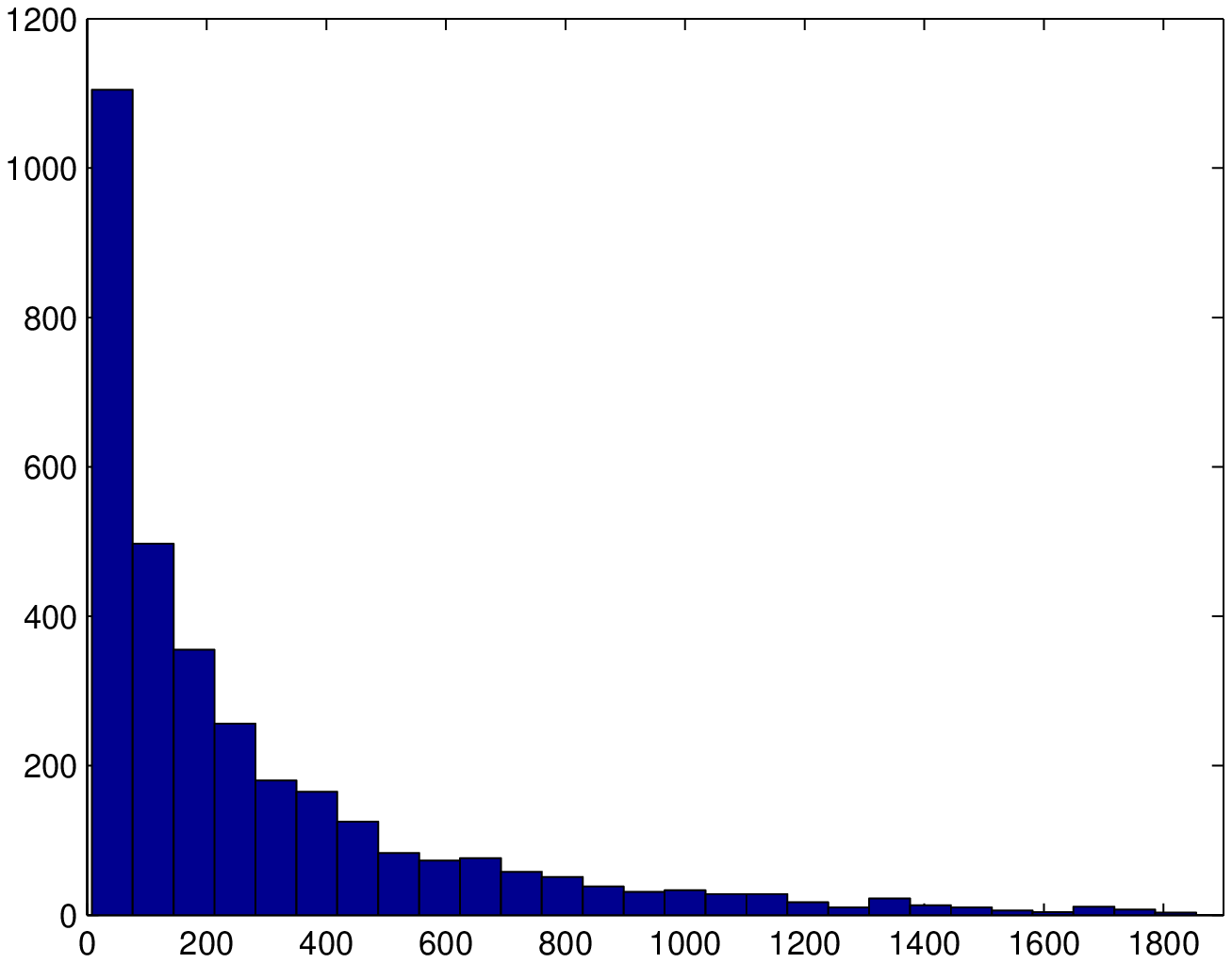}  & 
\includegraphics[scale=0.34]{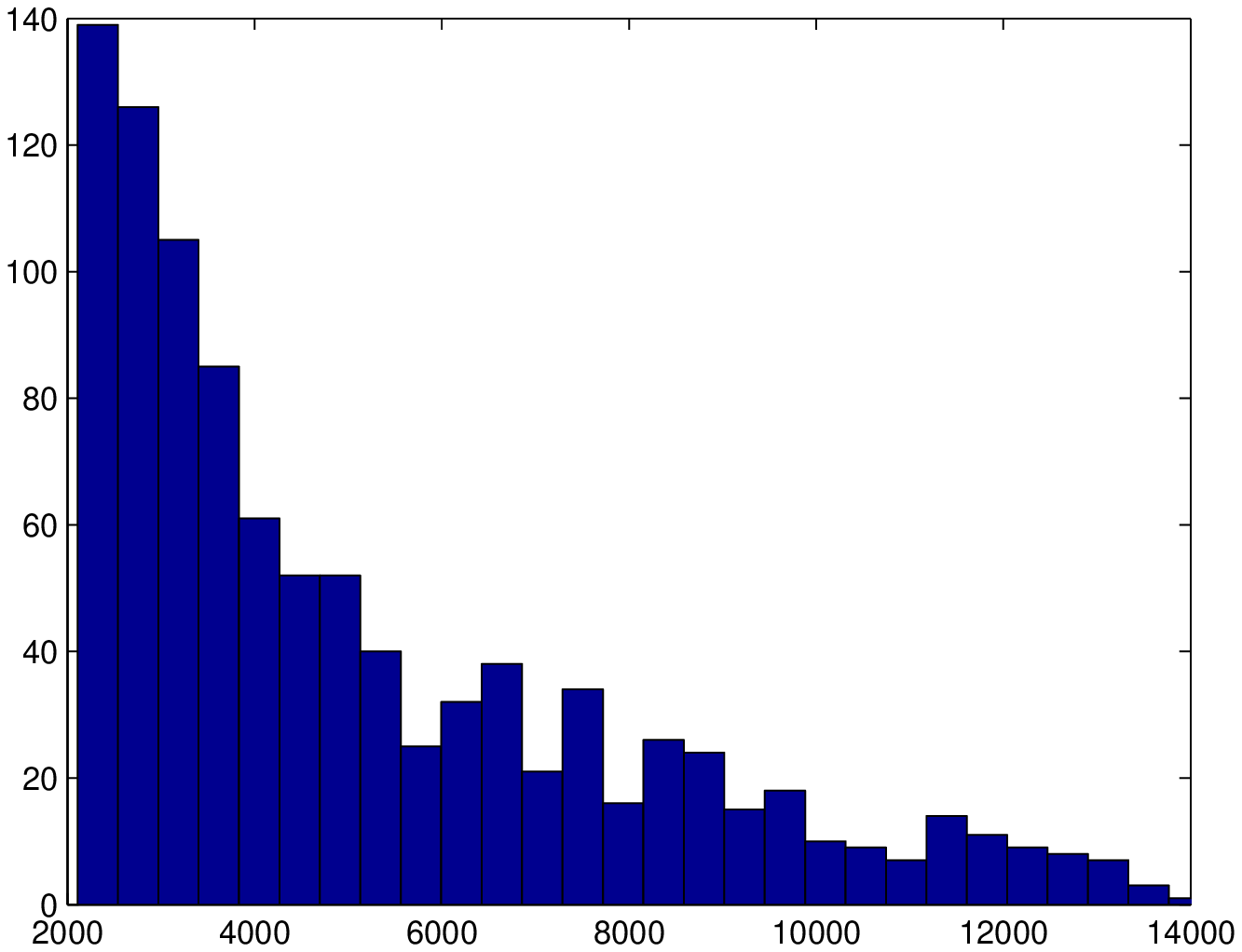}  &
\includegraphics[scale=0.34]{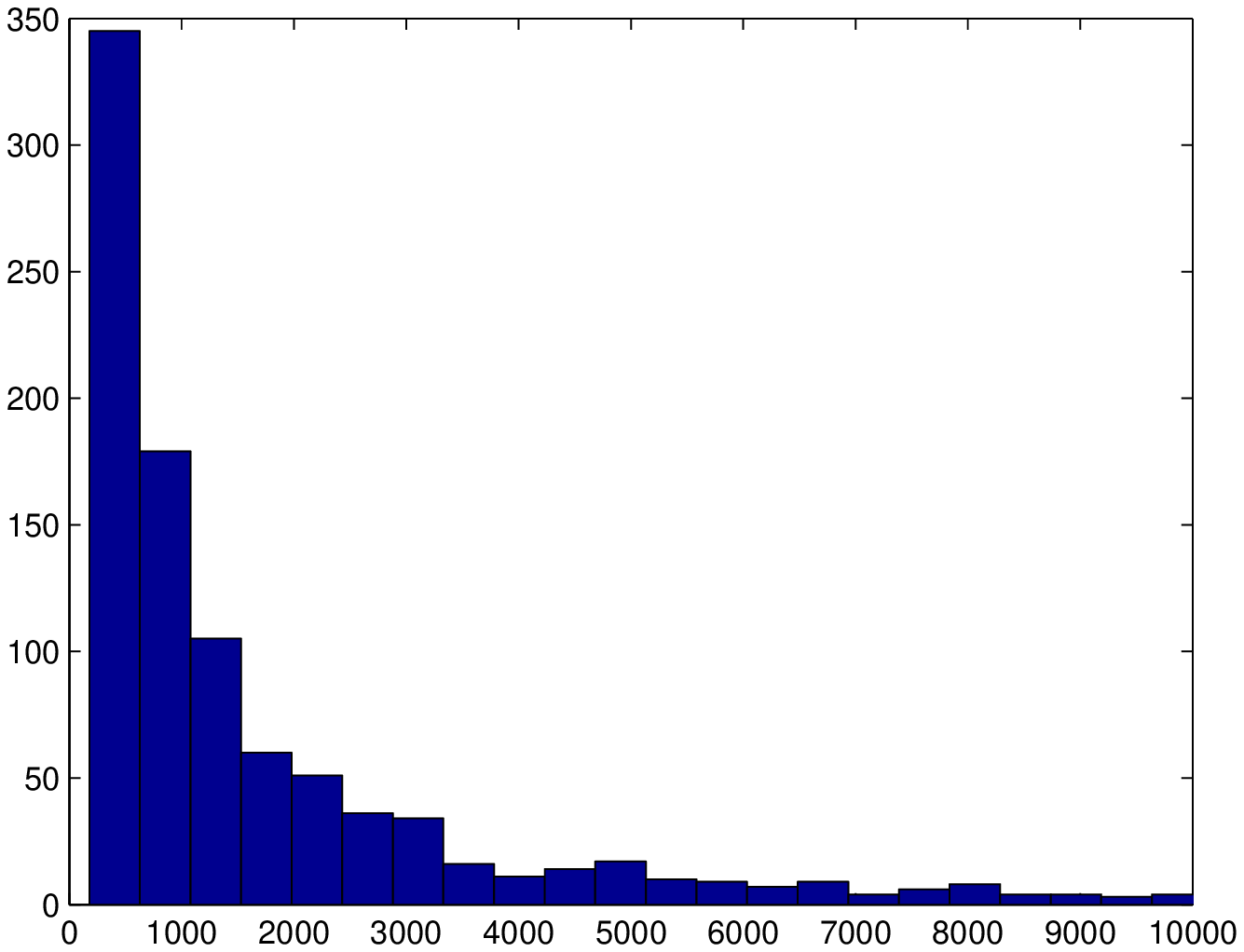}  \\
\includegraphics[scale=0.34]{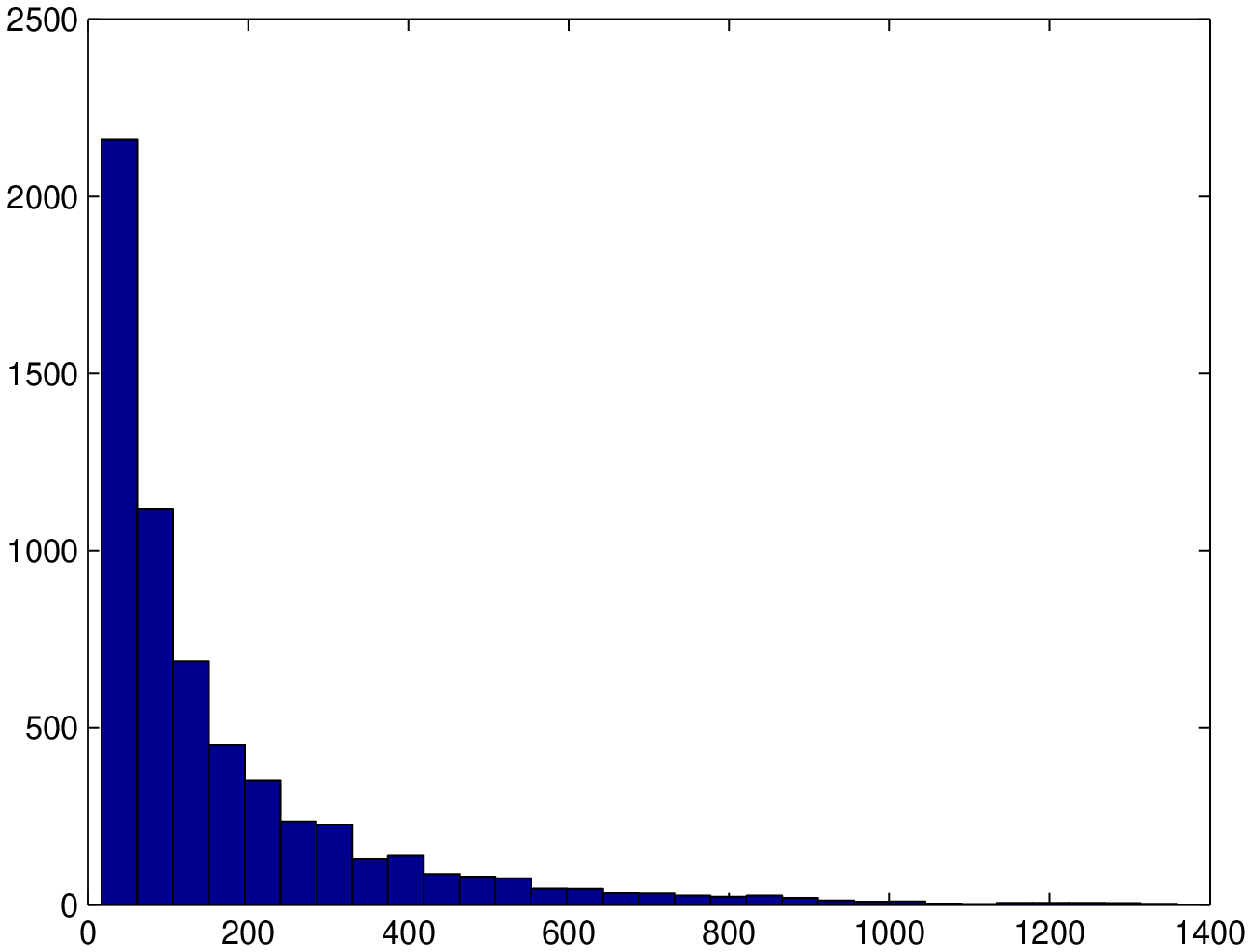}   & 
\includegraphics[scale=0.34]{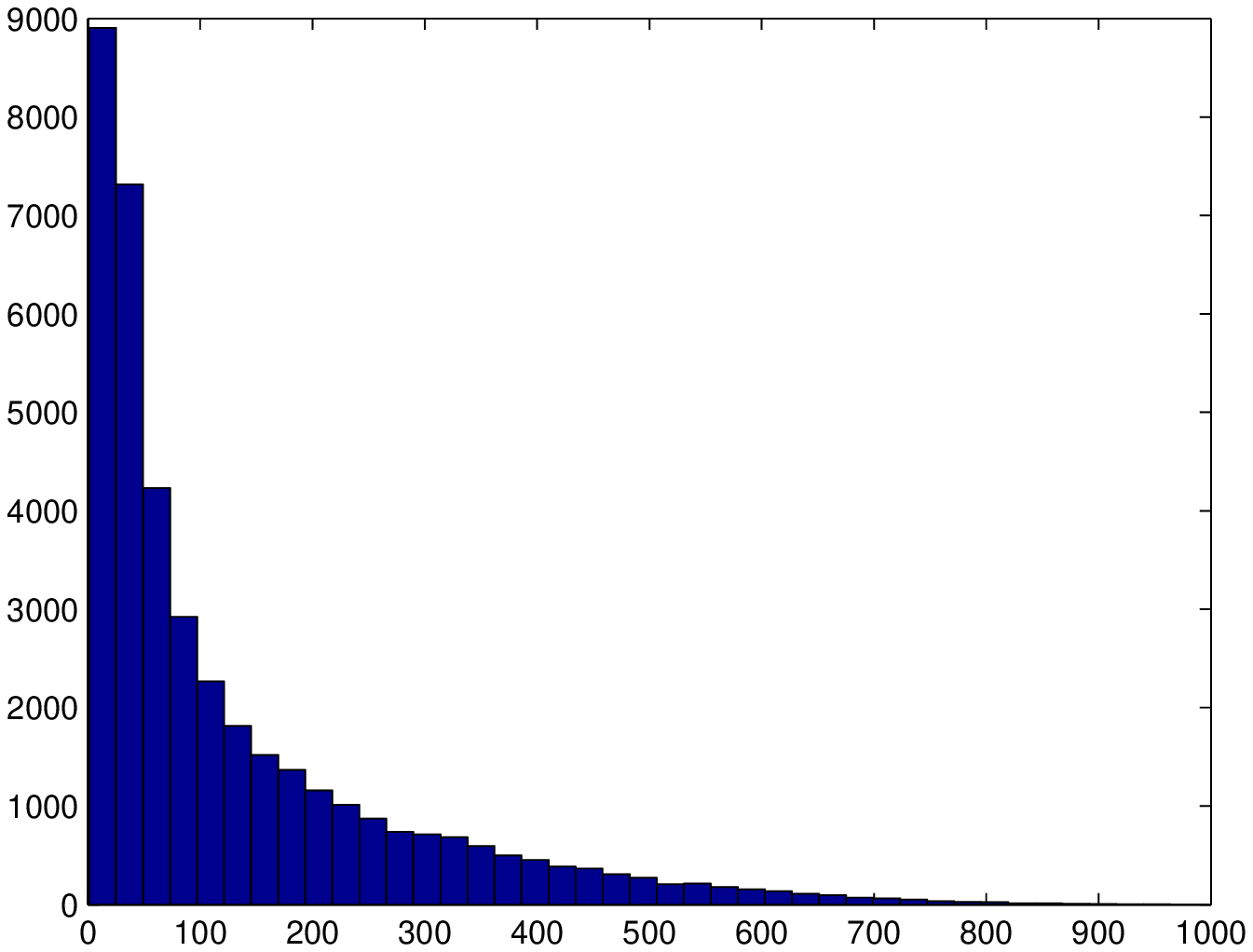}   &
\includegraphics[scale=0.34]{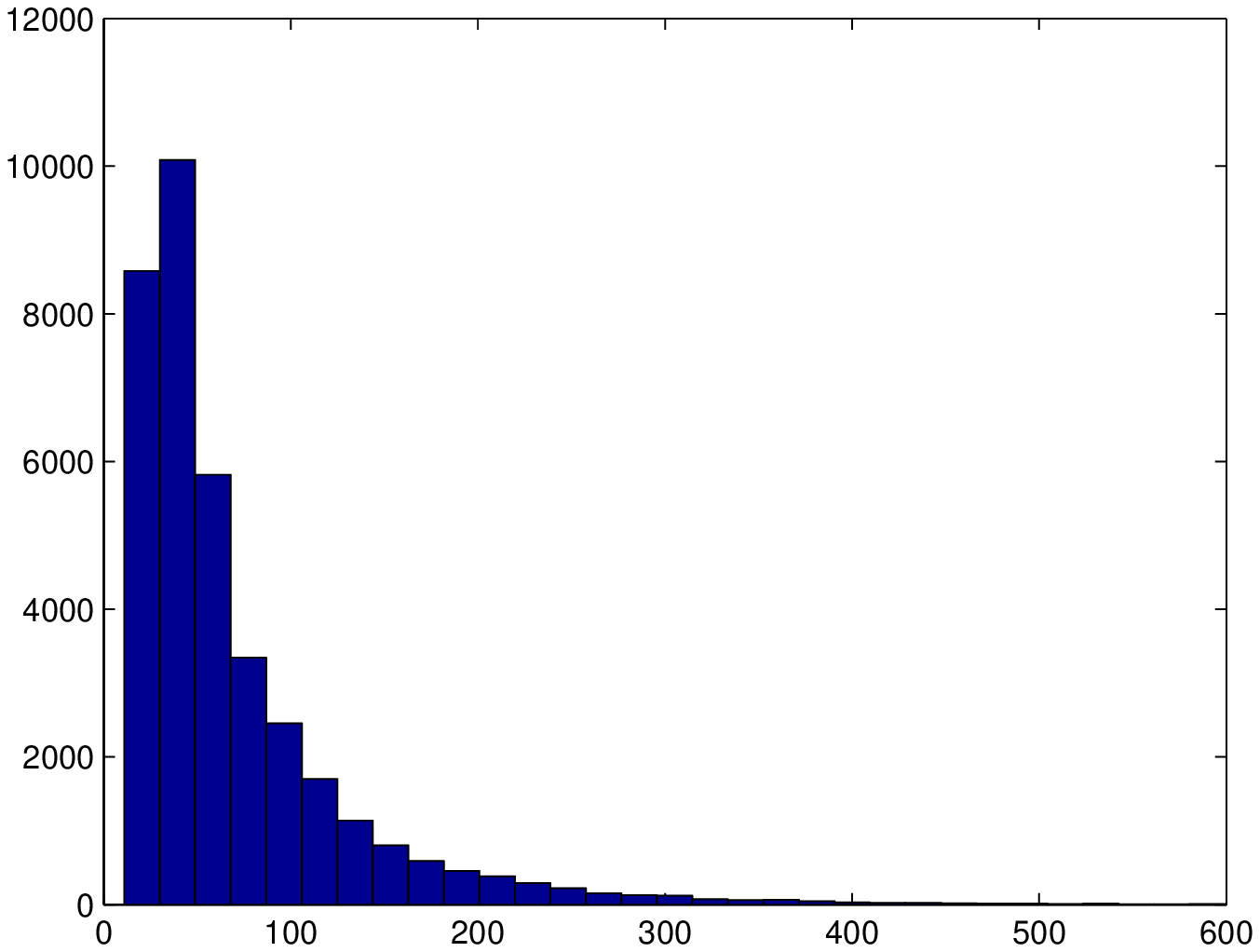}    \\
\includegraphics[scale=0.34]{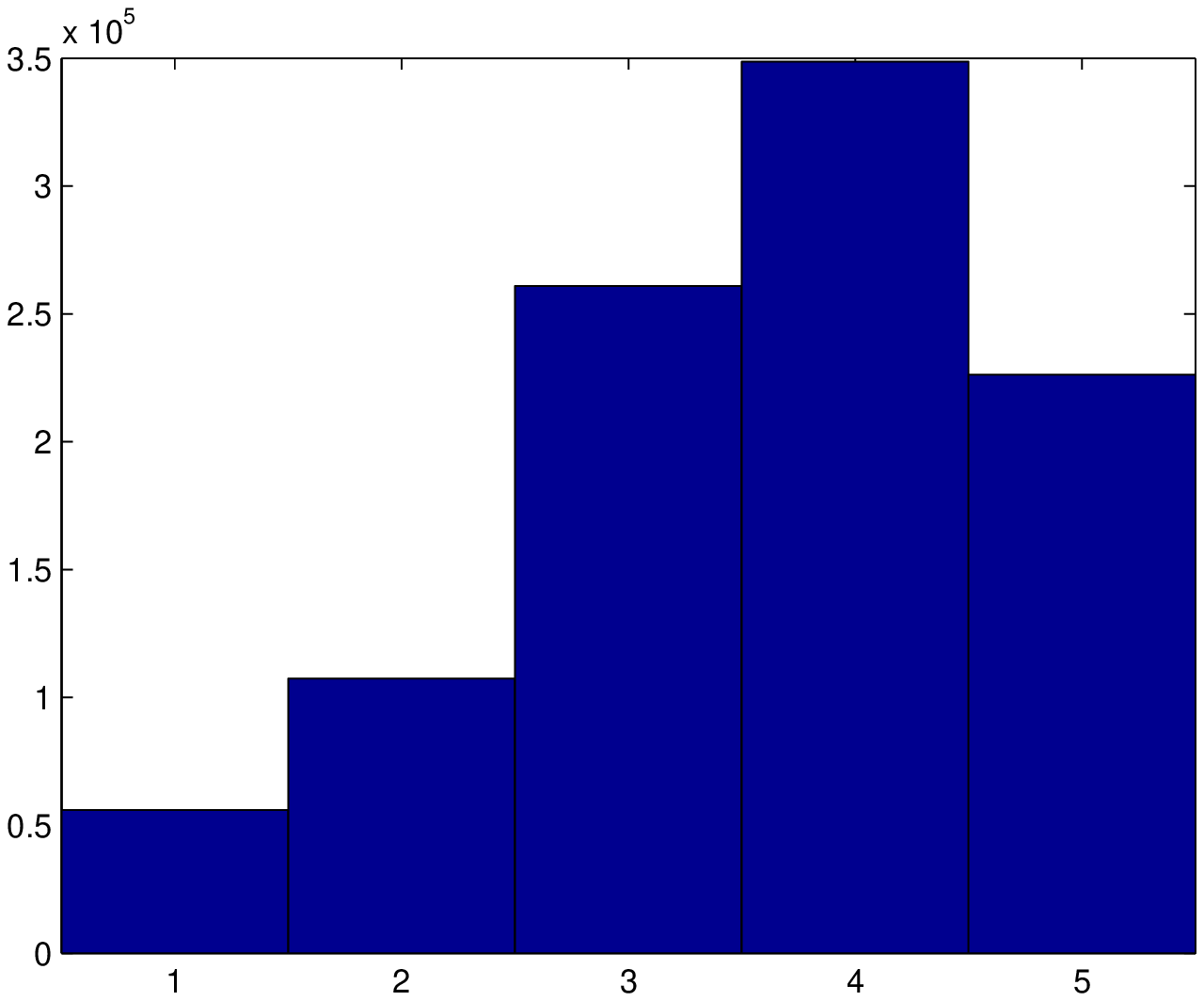}  &   
\includegraphics[scale=0.34]{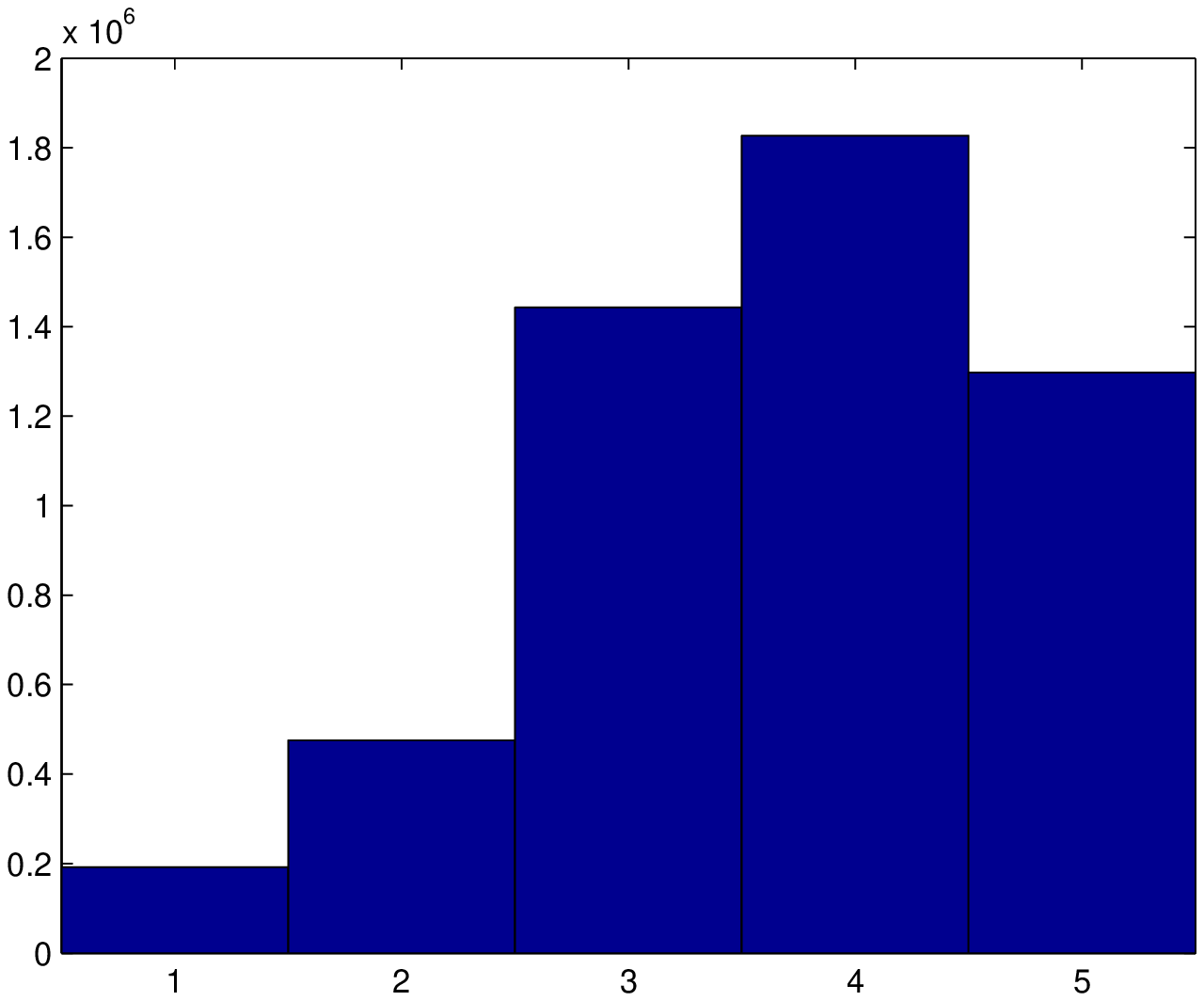}    &
\includegraphics[scale=0.34]{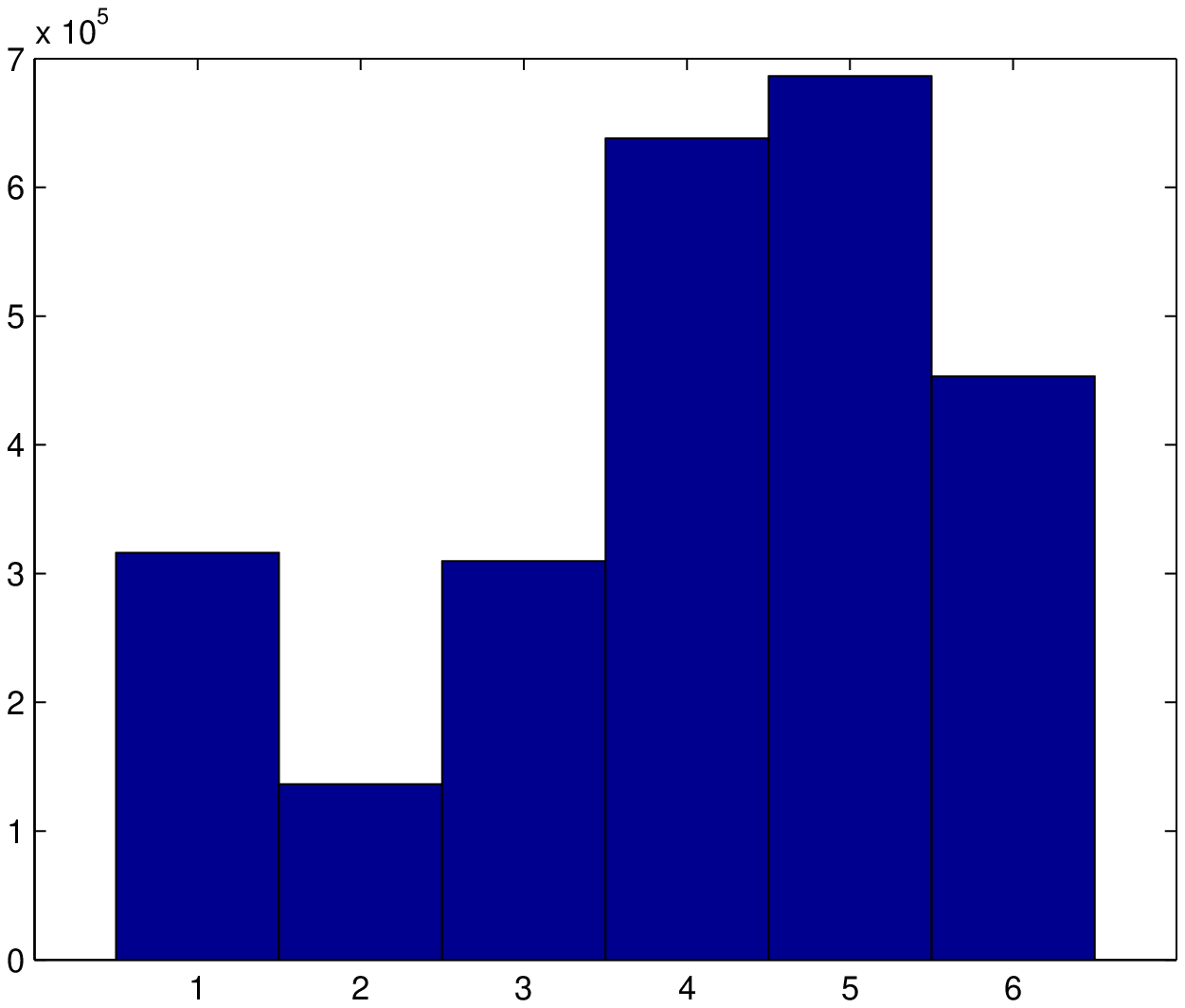}
\end{tabular}
\caption{Histograms of the number of user votes per movie (top row), number of movies ranked per user (middle row),  and votes (bottom row). 
}
\label{fig:hists} 
\vspace{-0.15in}
\end{figure}

\subsection{Estimating Probabilities} 

We consider here the task of estimating  $\hat p(R)$ where $R$ is a set of permutations corresponding to a tied incomplete ranking.  Such estimates may be used to compute conditional estimates $\hat P(R|S_{m+1})$ which are used to predict which augmentations $R$ of $S_{m+1}$ are highly probable. For example, given an observed preference $3\prec 2 \prec 5$ we may want to compute $\hat p(8 \prec 3 \prec 2\prec 5 | 3\prec 2 \prec 5)=\hat p(8\prec 3\prec 2 \prec 5) / \hat p(3 \prec 2\prec 5)$ to see whether item $8$ should be recommended to the user.

For simplicity we focus in this section on probabilities of simple events such as $i\prec j$ or $i\prec j \prec k$. The next section deals with more complex events. In our experiment, we estimate the probability of $i\prec j$ for the $n=53$ most rated movies in Netflix and $m=10000$ users who rate most of these movies. The probability matrix of the pairs is shown in Figure~\ref{fig:pairPrb1} where each cell corresponds to the probability of preference between a pair of movies determined by row $j$ and column $i$. In the top left panel the rows and columns are ordered by average probability of a movie being preferred to others $r(i)=\frac{\sum_{j}\hat p(i\prec j)}{n}$ with the most preferred movie in row and column 1 (top right panel indicates the ordering according to $r(i)$). In the bottom left panel the movies were ordered first by popularity of genres and then by $r(i)$. The bottom right panel indicates that ordering. The names, genres, and both orderings of all 53 movies appear in Figure \ref{fig:netflix53stat}.

The three highest movies in terms of $r(i)$ are Lord of the Rings: The Return of the King, Finding Nemo, and Lord of the Rings: The Two Towers. The three lowest movies are  Maid in Manhattan, Anger Management, and The Royal Tenenbaums. Examining the genre (colors in right panels of Figure~\ref{fig:pairPrb1}) we see that family and science fiction are generally preferred to others movies while comedy and romance generally receive lower preferences. The drama, action genres are somewhere in the middle. 

Also interesting is the variance of the movie preferences within specific genres. Family movies are generally preferred to almost all other movies.  Science fiction movies, on the other hand, enjoy high preference overall but exhibit a larger amount of variability as a few movies are among the least preferred. Similarly, the preference probabilities of action movies are widely spread with some movies being preferred to others and others being less preferred. More specifically (see bottom left panel of Figure~\ref{fig:pairPrb1}) we see that the decay of $r(i)$ within genres is linear for family and romance and nonlinear for science fiction, action, drama, and comedy. In these last three genres there are a few really ``bad'' movies that are substantially lower than the rest of the curve. Figure~\ref{fig:netflix53stat} shows the full information including titles, genres and orderings of the $53$ most popular movies in Netflix.

\begin{figure}\centering
\begin{tabular}{cc}
\includegraphics[scale=0.55]{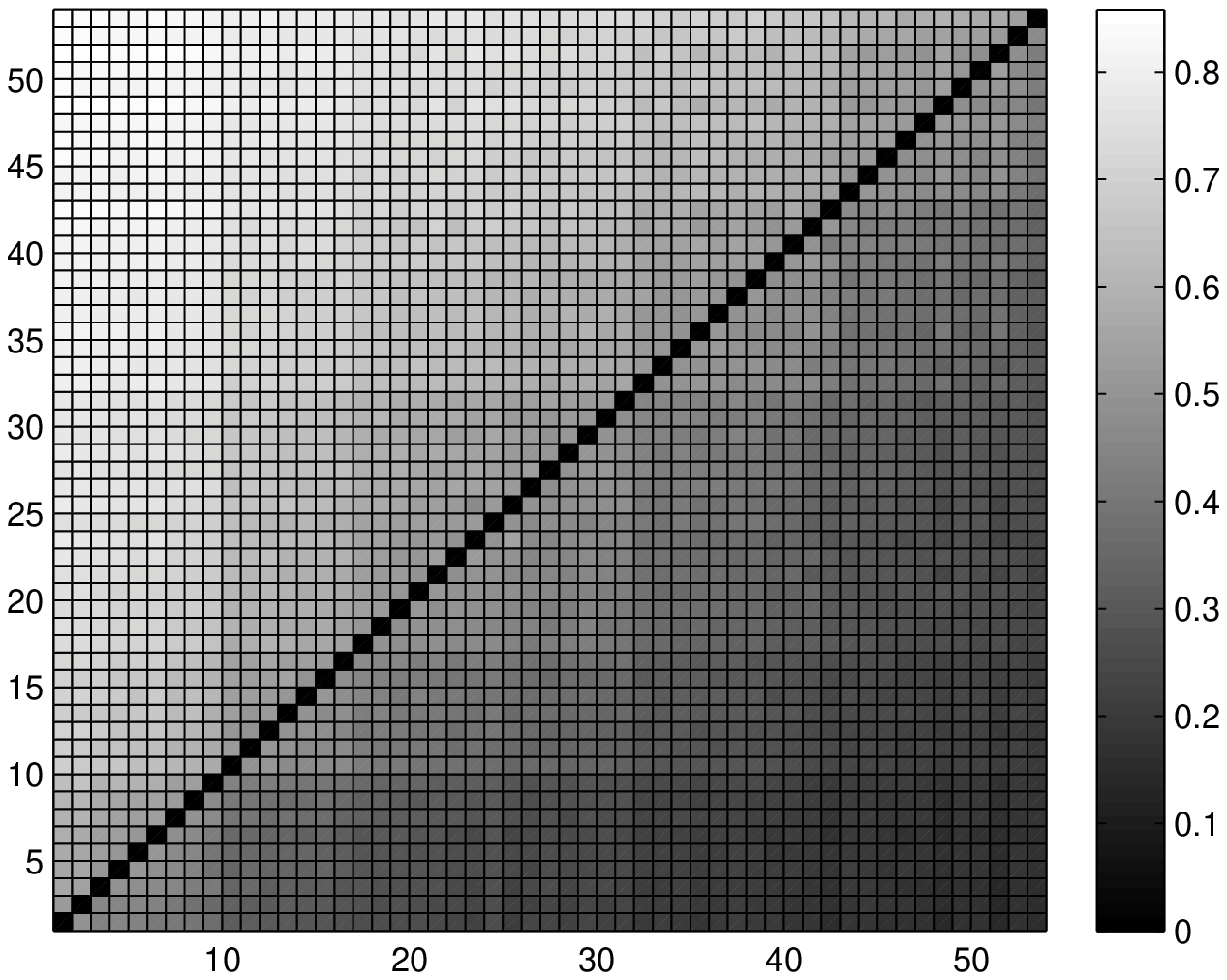}\hspace{-0.5in} &
\includegraphics[scale=0.55]{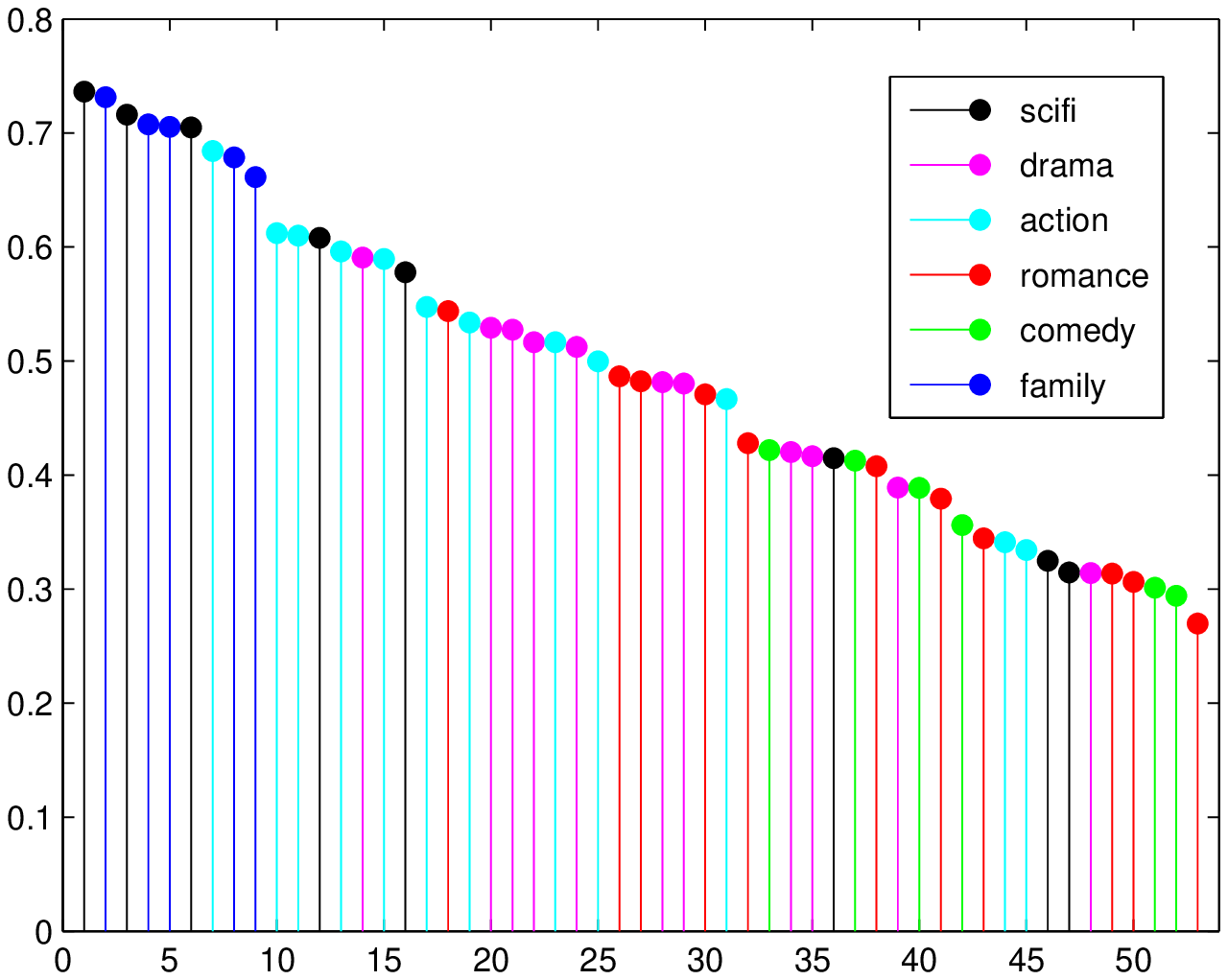}\\
\includegraphics[scale=0.55]{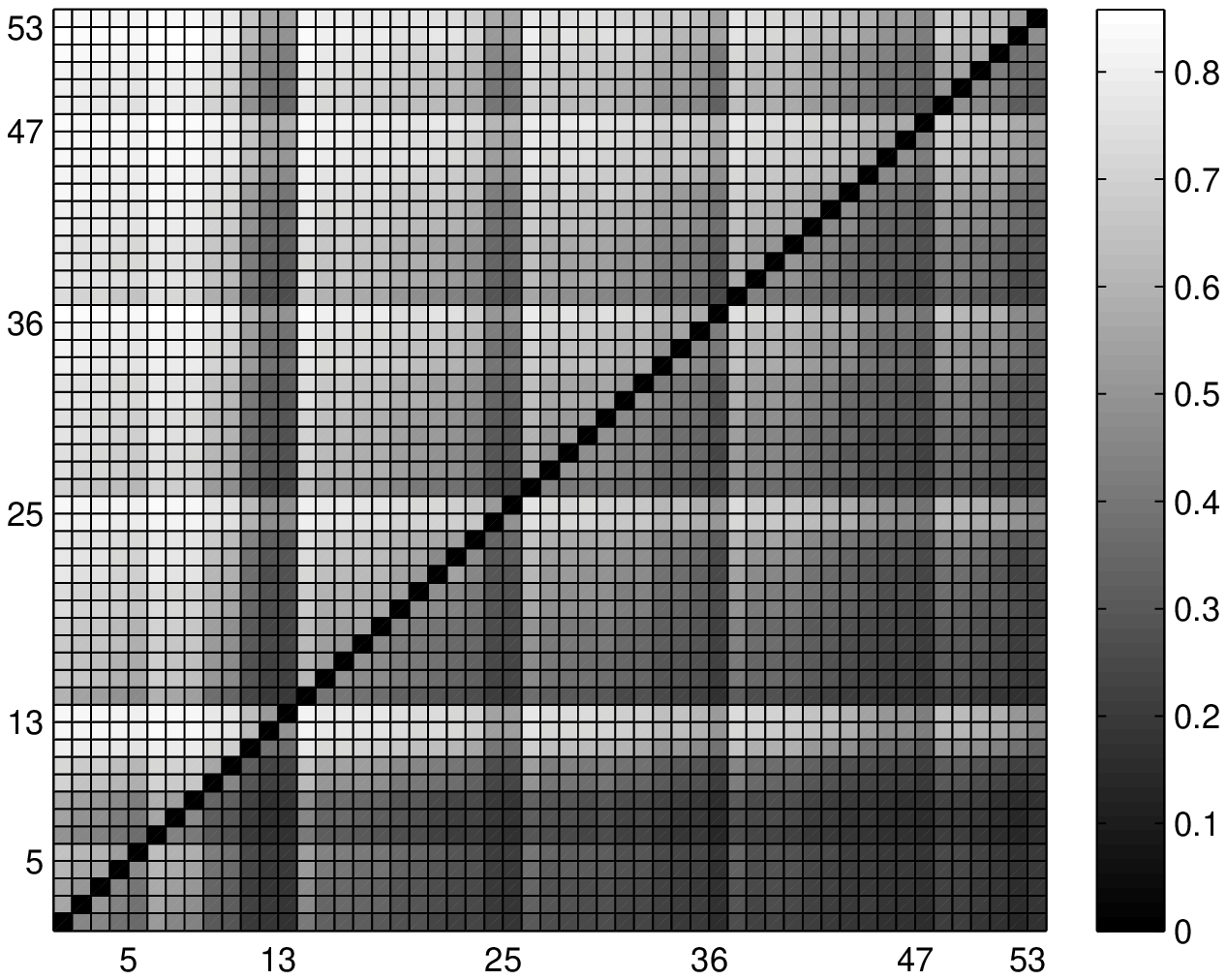}\hspace{-0.5in} &
\includegraphics[scale=0.55]{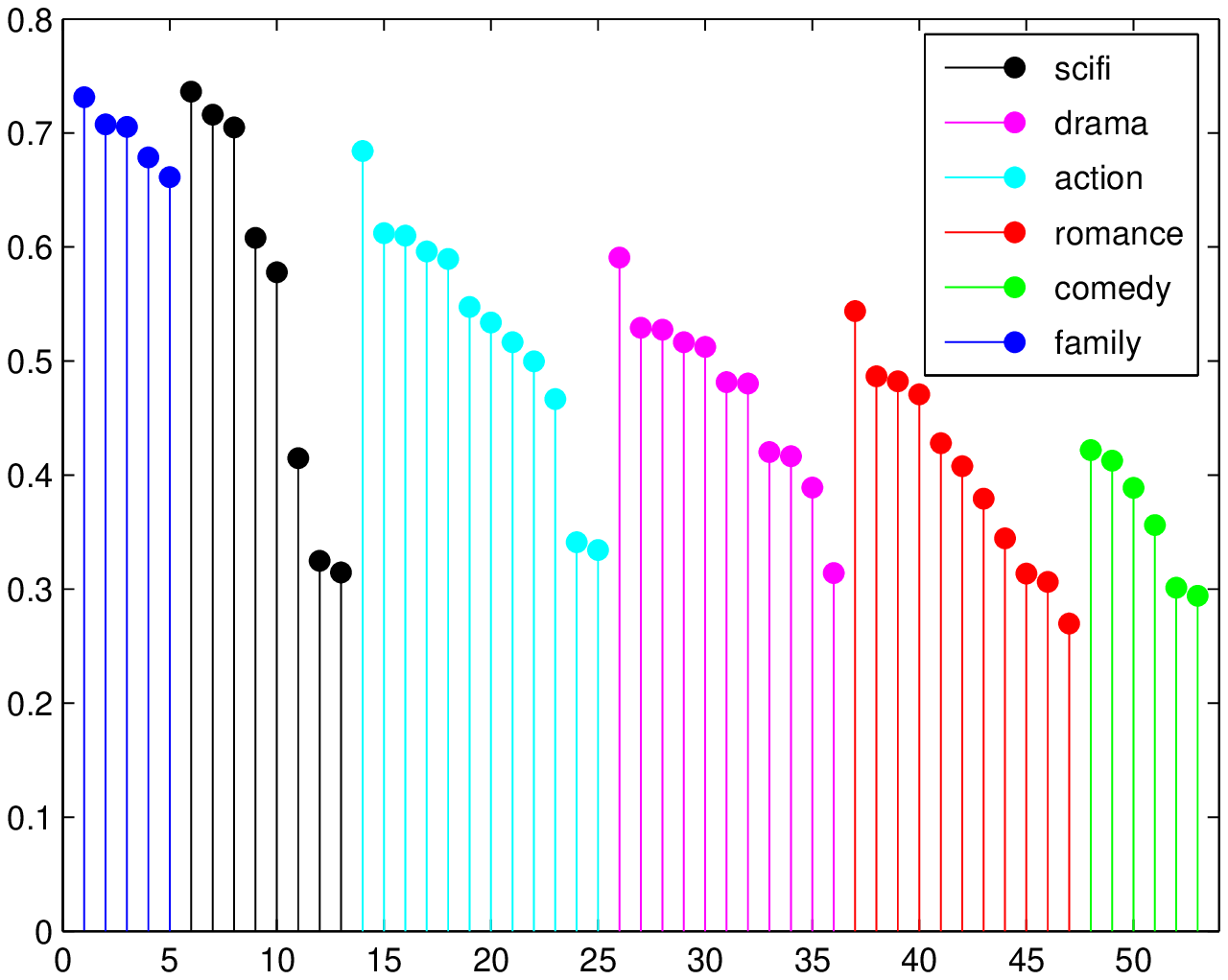}\\
\end{tabular}\vspace{-0.2in}
\caption{Left: The estimated probability of movie $i$ being preferred to movie $j$. Right: a plot of $r(i)=\sum_j \hat p(i\prec j)/n$ for all movies with color indicating genres. In both panels the movies were ordered by $r(i)$ (top row) and first by popularity of genres and then by $r(i)$ (bottom row).}  
\label{fig:pairPrb1}
\end{figure}

We plot the individual values of $\hat p(i\prec j)$ for three movies:  Shrek (family), Catch Me If You Can (drama) and Napoleon Dynamite (comedy) (Figure~\ref{fig:pairPrb3}). Comparing the three stem plots we observe that Shrek is preferred to almost all other movies, Napoleon Dynamite is less preferred than most other movies, and Catch Me If You Can is preferred to some other movies but less preferred than others. Also interesting is the linear increase of the stem plots for Catch Me If You Can and Napoleon Dynamite and the non-linear increase of the stem plot for Shrek. This is likely a result of the fact that for very popular movies there are only a few comparable movies with the rest being very likely to be less preferred movies ($\hat p(i\prec j)$ close to 1).

\begin{figure}\centering
\begin{tabular}{ccc}\hspace{-0.4in}
\includegraphics[scale=0.41]{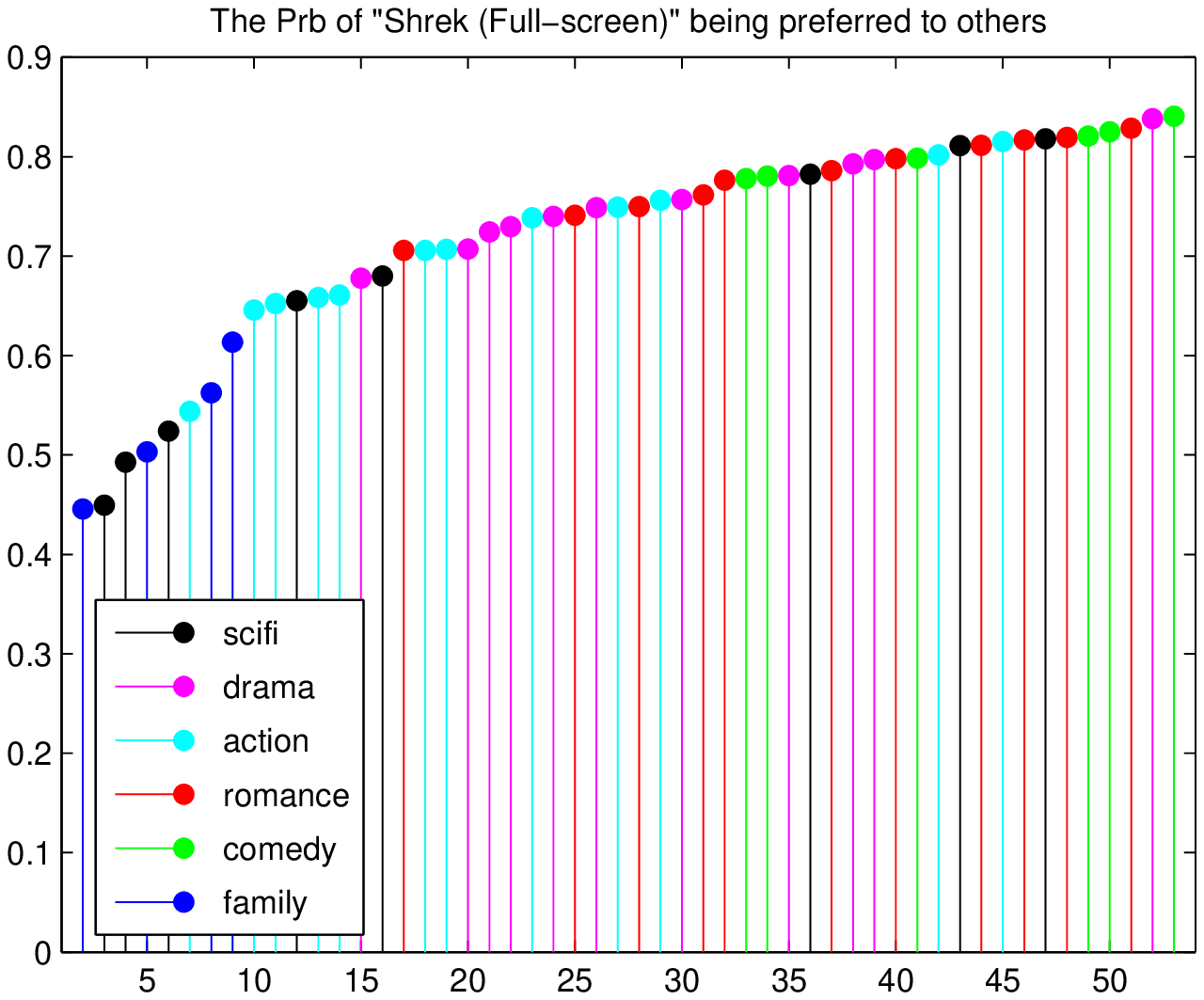}\hspace{-0.4in} &
\includegraphics[scale=0.41]{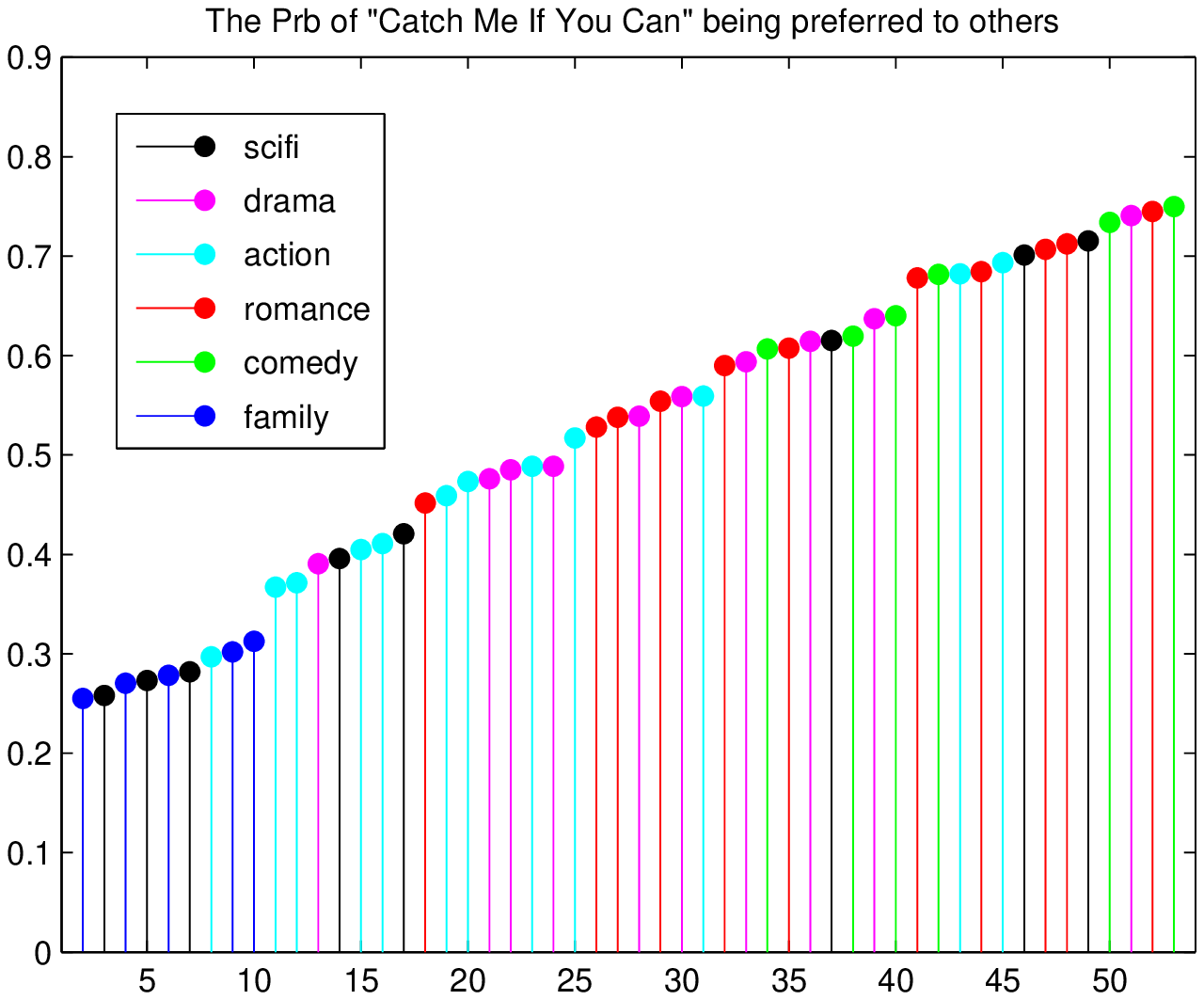}\hspace{-0.4in} &
\includegraphics[scale=0.41]{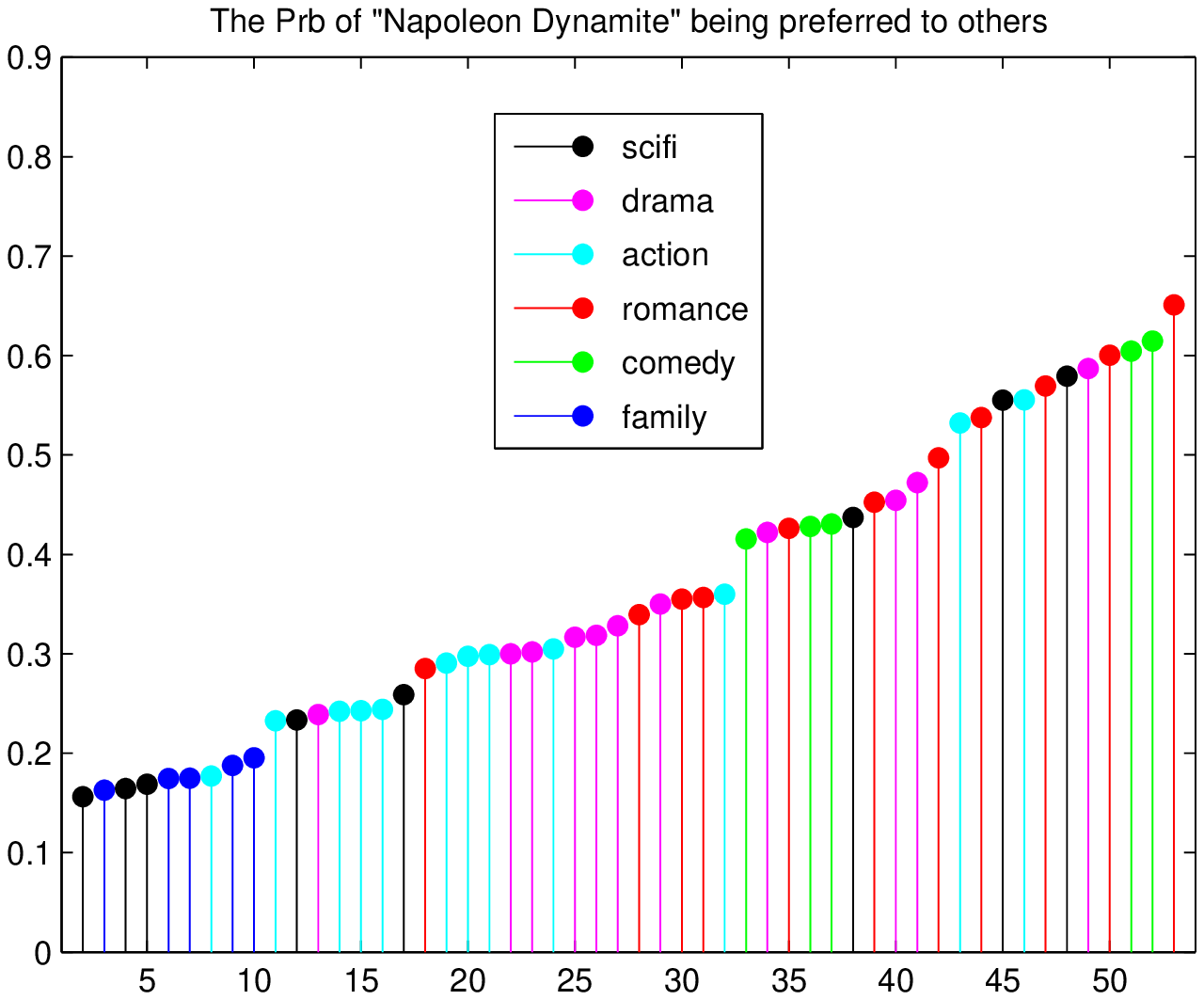}\\
\end{tabular}
\caption{The value $\hat p(i\prec j)$ for all $j$ for three movies: Shrek (left), Catch Me If You Can (middle) and Napoleon Dynamite (right).}
\label{fig:pairPrb3}
\end{figure}

\begin{figure}
\begin{tabular}{ll}
\begin{tabular}{|llll|}
\hline 
\scriptsize Titles & \tiny Genre & \tiny  Order1 & \tiny  Order2 \\ \hline 
\scriptsize Finding Nemo &\scriptsize \textcolor{blue}6 & \scriptsize 2 & \scriptsize 1 \\ \hline  
\scriptsize Shrek &\scriptsize \textcolor{blue}6 & \scriptsize 4 & \scriptsize 2 \\ \hline  
\scriptsize The Incredibles &\scriptsize \textcolor{blue}6 & \scriptsize 5 & \scriptsize 3 \\ \hline  
\scriptsize Monsters, Inc. &\scriptsize \textcolor{blue}6 & \scriptsize 8 & \scriptsize 4 \\ \hline  
\scriptsize Shrek II & \scriptsize \textcolor{blue}6 & \scriptsize 9 & \scriptsize 5 \\ \hline  
\scriptsize LOTR: The Return of the King &\scriptsize \textcolor{black}1 & \scriptsize 1 & \scriptsize 6 \\ \hline 
\scriptsize LOTR: The Two Towers &\scriptsize \textcolor{black}1 & \scriptsize 3 & \scriptsize 7 \\ \hline
\scriptsize LOTR: The Fellowship of the Ring &\scriptsize \textcolor{black}1 & \scriptsize 6 & \scriptsize 8 \\ \hline 
\scriptsize Spider-Man II &\scriptsize \textcolor{black}1 & \scriptsize 12 & \scriptsize 9 \\ \hline  
\scriptsize Spider-Man &\scriptsize \textcolor{black}1 & \scriptsize 16 & \scriptsize 10 \\ \hline  
\scriptsize The Day After Tomorrow &\scriptsize \textcolor{black}1 & \scriptsize 36 & \scriptsize 11\\ \hline  
\scriptsize Tomb Raider &\scriptsize \textcolor{black}1 & \scriptsize 46 & \scriptsize 12 \\ \hline  
\scriptsize Men in Black II &\scriptsize \textcolor{black}1 & \scriptsize 47 & \scriptsize 13 \\ \hline  
\scriptsize Pirates of the Caribbean I&\scriptsize \textcolor{cyan}3 & \scriptsize 7 & \scriptsize 14 \\ \hline  
\scriptsize The Last Samurai &\scriptsize \textcolor{cyan}3 & \scriptsize 10 & \scriptsize 15 \\ \hline  
\scriptsize Man on Fire& \scriptsize \textcolor{cyan}3 & \scriptsize 11 & \scriptsize 16 \\ \hline  
\scriptsize The Bourne Identity &\scriptsize \textcolor{cyan}3 & \scriptsize 13 & \scriptsize 17 \\ \hline  
\scriptsize The Bourne Supremacy &\scriptsize \textcolor{cyan}3 & \scriptsize 15 & \scriptsize 18 \\ \hline  
\scriptsize National Treasure &\scriptsize \textcolor{cyan}3 & \scriptsize 17 & \scriptsize 19 \\ \hline 
\scriptsize The Italian Job &\scriptsize \textcolor{cyan}3 & \scriptsize 19 & \scriptsize 20 \\ \hline  
\scriptsize Kill Bill II &\scriptsize \textcolor{cyan}3 & \scriptsize 23 & \scriptsize 21\\ \hline  
\scriptsize Kill Bill I &\scriptsize \textcolor{cyan}3 & \scriptsize 25 & \scriptsize 22\\ \hline  
\scriptsize Minority Report &\scriptsize \textcolor{cyan}3 & \scriptsize 31 & \scriptsize 23\\ \hline  
\scriptsize S.W.A.T. &\scriptsize \textcolor{cyan}3 & \scriptsize 44 & \scriptsize 24 \\ \hline  
\scriptsize The Fast and the Furious &\scriptsize \textcolor{cyan}3 & \scriptsize 45 & \scriptsize 25 \\ \hline  
\scriptsize Ocean's Eleven &\scriptsize \textcolor{magenta}2 & \scriptsize 14 & \scriptsize 26 \\ \hline  
\scriptsize I, Robot &\scriptsize \textcolor{magenta}2 & \scriptsize 20 & \scriptsize 27 \\ \hline  
\end{tabular} & \hspace{-0.15in}
\begin{tabular}{|llll|}
\hline
\scriptsize Titles & \tiny Genre & \tiny  Order1 & \tiny Order2 \\ \hline  
\scriptsize Mystic River &\scriptsize \textcolor{magenta}2 & \scriptsize 21 & \scriptsize 28 \\ \hline  
\scriptsize Troy &\scriptsize \textcolor{magenta}2 & \scriptsize 22 & \scriptsize 29 \\ \hline   
\scriptsize Catch Me If You Can &\scriptsize \textcolor{magenta}2 & \scriptsize 24 & \scriptsize 30\\ \hline  
\scriptsize Big Fish &\scriptsize \textcolor{magenta}2 & \scriptsize 28 & \scriptsize 31\\ \hline  
\scriptsize Collateral &\scriptsize \textcolor{magenta}2 & \scriptsize 29 & \scriptsize 32\\ \hline
\scriptsize John Q &\scriptsize \textcolor{magenta}2 & \scriptsize 34 & \scriptsize 33\\ \hline  
\scriptsize Pearl Harbor &\scriptsize \textcolor{magenta}2 & \scriptsize 35 & \scriptsize 34\\ \hline  
\scriptsize Swordfish &\scriptsize \textcolor{magenta}2 & \scriptsize 39 & \scriptsize 35\\ \hline 
\scriptsize Lost in Translation &\scriptsize\textcolor{magenta}2 & \scriptsize 48 & \scriptsize 36 \\ \hline  
\scriptsize 50 First Dates &\scriptsize \textcolor{red}4 & \scriptsize 18 & \scriptsize 37 \\ \hline  
\scriptsize My Big Fat Greek Wedding &\scriptsize  \textcolor{red}4 & \scriptsize 26 & \scriptsize 38\\ \hline  
\scriptsize Something's Gotta Give &\scriptsize  \textcolor{red}4 & \scriptsize 27 & \scriptsize 39\\ \hline  
\scriptsize The Terminal &\scriptsize  \textcolor{red}4 & \scriptsize 30 & \scriptsize 40\\ \hline  
\scriptsize How to Lose a Guy in 10 Days &\scriptsize  \textcolor{red}4 & \scriptsize 32 & \scriptsize 41\\ \hline  
\scriptsize Sweet Home Alabama &\scriptsize  \textcolor{red}4 & \scriptsize 38 & \scriptsize 42\\ \hline 
\scriptsize Sideways &\scriptsize  \textcolor{red}4 & \scriptsize 41 & \scriptsize 43 \\ \hline  
\scriptsize Two Weeks Notice &\scriptsize  \textcolor{red}4 & \scriptsize 43 & \scriptsize 44 \\ \hline  
\scriptsize Mr. Deeds &\scriptsize  \textcolor{red}4 & \scriptsize 49 & \scriptsize 45 \\ \hline  
\scriptsize The Wedding Planner &\scriptsize  \textcolor{red}4 & \scriptsize 50 & \scriptsize 46 \\ \hline
\scriptsize Maid in Manhattan &\scriptsize  \textcolor{red}4 & \scriptsize 53 & \scriptsize 47 \\ \hline  
\scriptsize The School of Rock &\scriptsize \textcolor{green}5 & \scriptsize 33 & \scriptsize 48\\ \hline  
\scriptsize Bruce Almighty &\scriptsize \textcolor{green}5 & \scriptsize 37 & \scriptsize 49 \\ \hline   
\scriptsize Dodgeball: A True Underdog Story &\scriptsize \textcolor{green}5 & \scriptsize 40 & \scriptsize 50\\ \hline  
\scriptsize Napoleon Dynamite &\scriptsize \textcolor{green}5 & \scriptsize 42 & \scriptsize 51 \\ \hline  
\scriptsize The Royal Tenenbaums &\scriptsize \textcolor{green}5 & \scriptsize 51 & \scriptsize 52 \\ \hline  
\scriptsize Anger Management &\scriptsize \textcolor{green}5 & \scriptsize 52 & \scriptsize 53 \\ \hline    
\scriptsize  & &  &  \\ \hline   
\end{tabular} \\
\end{tabular}
\caption{The table contains the information of the $53$ most popular movies of Netflix. Columns are movie titles, genres, order1 (the ordering in the upper row of Figure~\ref{fig:pairPrb1}) and order2 (the ordering in the bottom row of Figure~\ref{fig:pairPrb1}). Genres indicated by numbers from $1$ to $6$ represent science fiction, drama, action, romance, comedy, and family.}
\label{fig:netflix53stat} \vspace{-.1in}
\end{figure}

In a second experiment (see Figure~\ref{fig:taskLoglikely}) we compare the predictive behavior of the kernel smoothing estimator with that of a parametric model (Mallows model) and the empirical measure (frequency of event occurring in the $m$ samples). We evaluate the predictive performance of a probability estimator by separating the data to two parts: a training set that is used to construct the estimator and a testing set used for evaluation via its loglikelihood. A higher test set loglikelihood indicates that the model assigns high probability to events that occurred. Mathematically, this corresponds to approximating the KL divergence between nature and the model. Since the Mallows model is intractable for large $n$ we chose in this experiment small values of $n$: $3,4,5$. 

We observe that the kernel estimator consistently achieves higher test set loglikelihood than the Mallows model and the empirical measure. The former is due to the breakdown of parametric assumptions as indicated by Figure~\ref{fig:heatMaps} (note that this happens even for $n$ as low as 3). The latter is due to the superior statistical performance of the kernel estimator over the empirical measure. 

\begin{figure}\centering
\begin{tabular}{ccc}
{\scriptsize Movielens1M}& {\scriptsize Netflix} &{\scriptsize EachMovie}    \\
\includegraphics[scale=0.31]{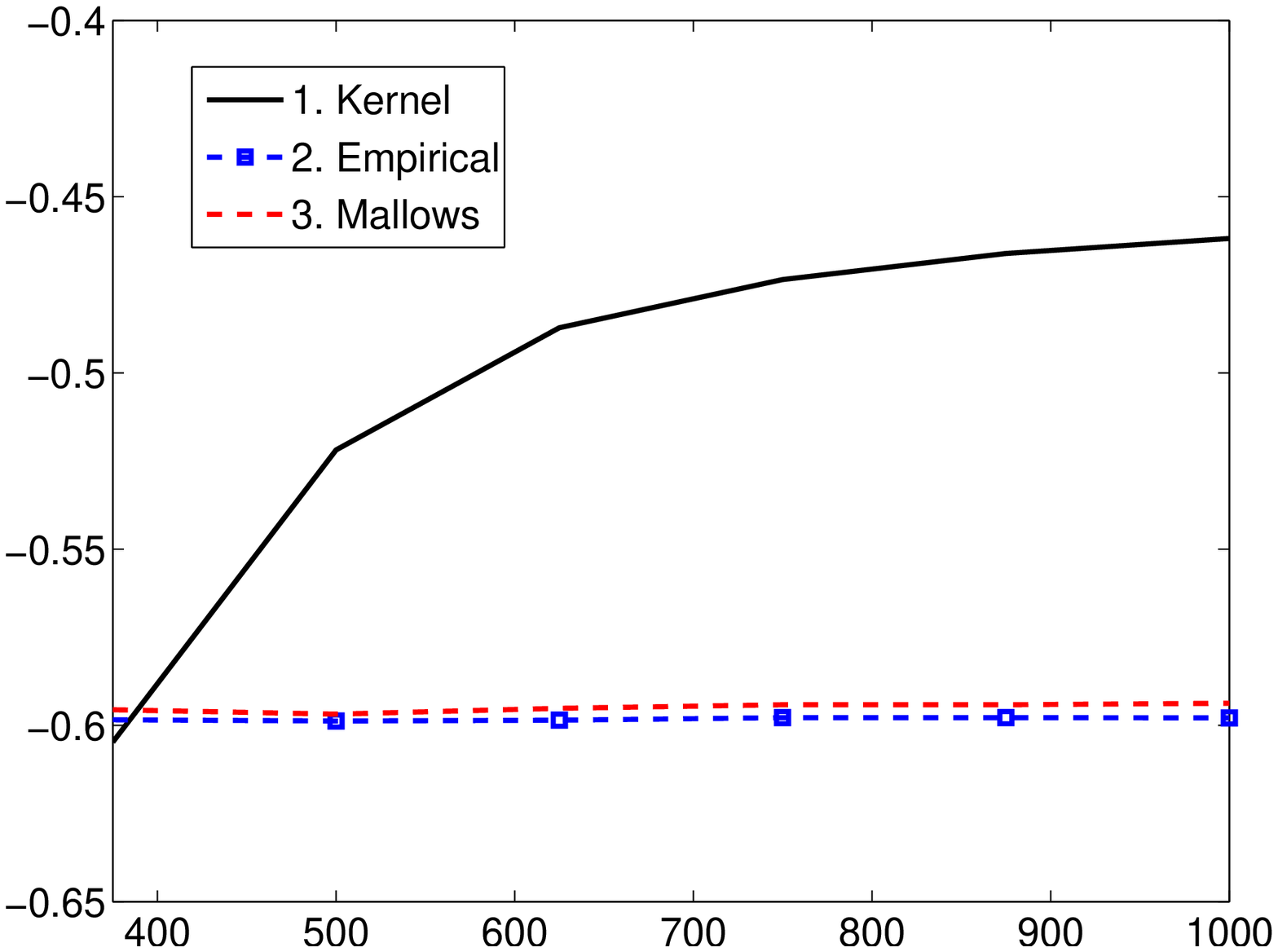}\hspace{-0.1in} &
\includegraphics[scale=0.31]{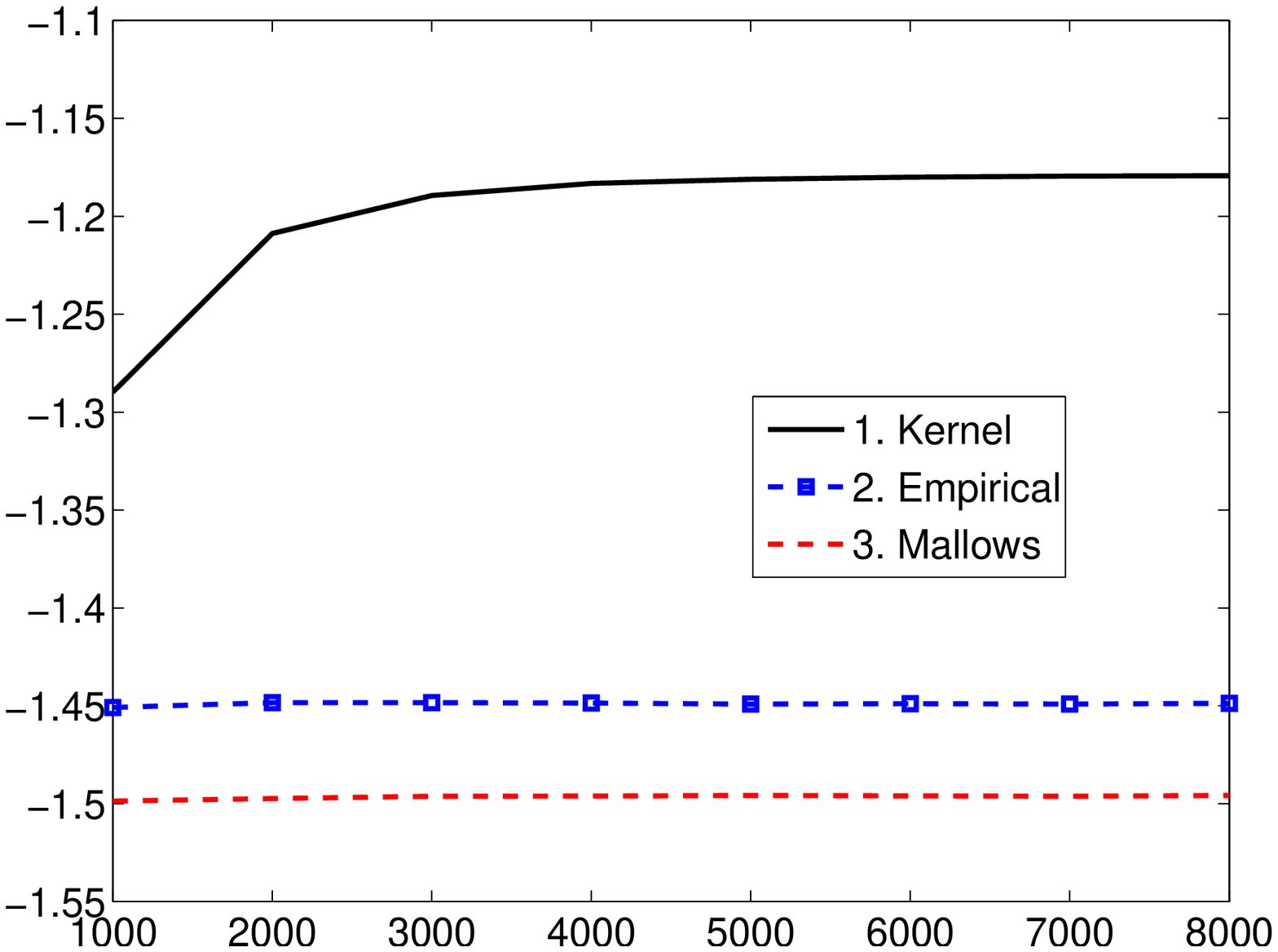}\hspace{-0.1in} &
\includegraphics[scale=0.31]{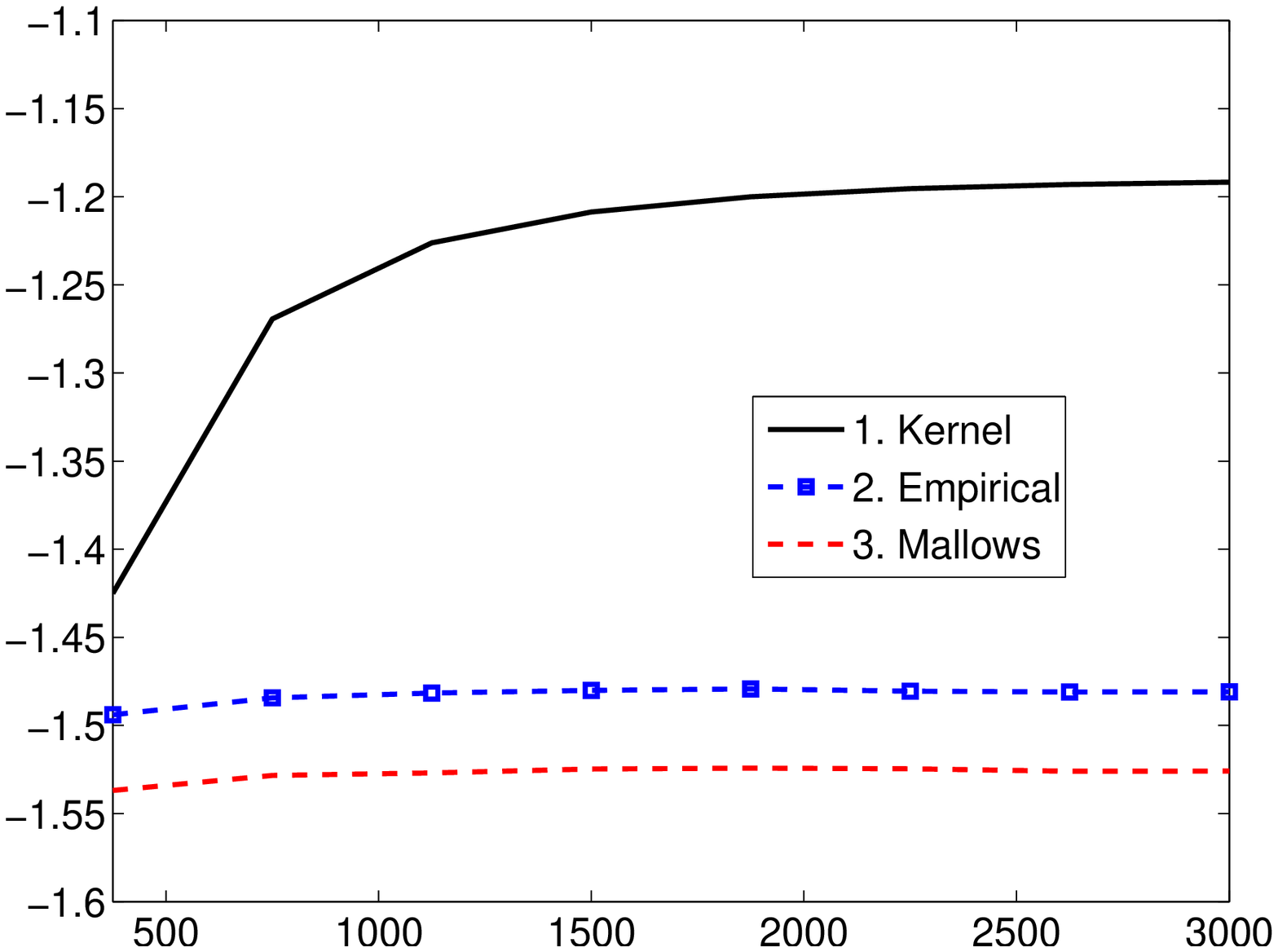} \\
\includegraphics[scale=0.31]{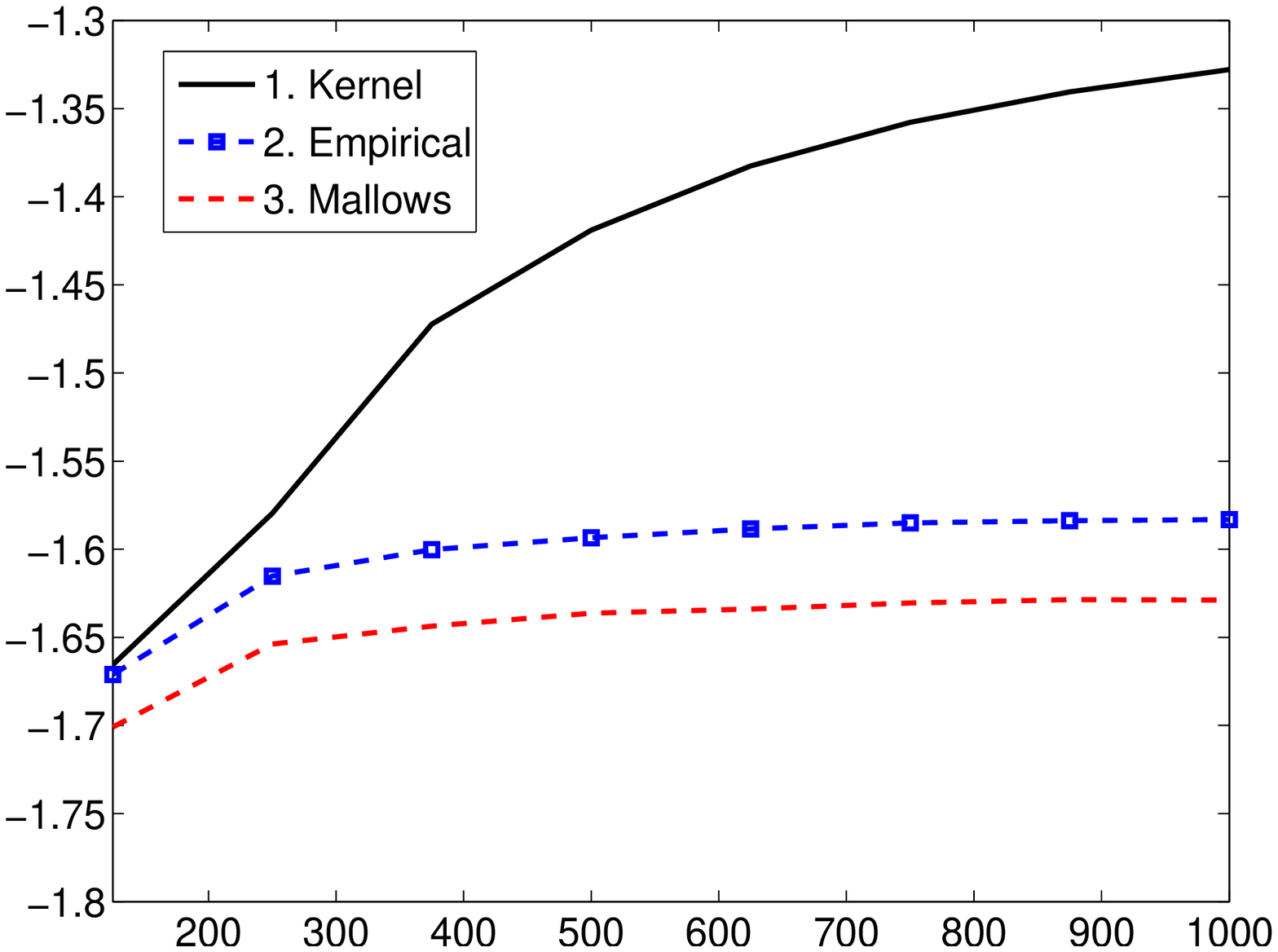}\hspace{-0.1in} &
\includegraphics[scale=0.31]{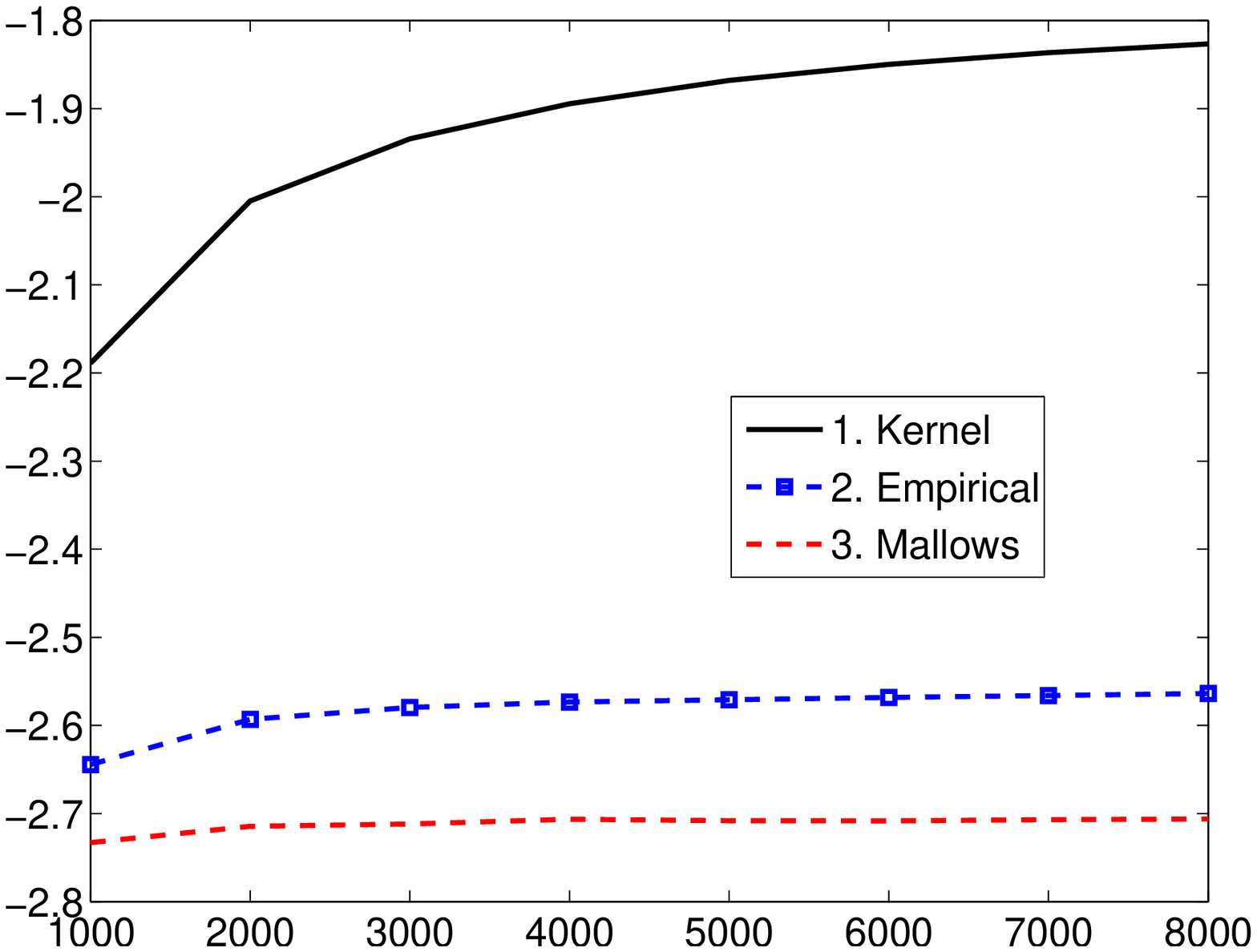}\hspace{-0.1in} &
\includegraphics[scale=0.31]{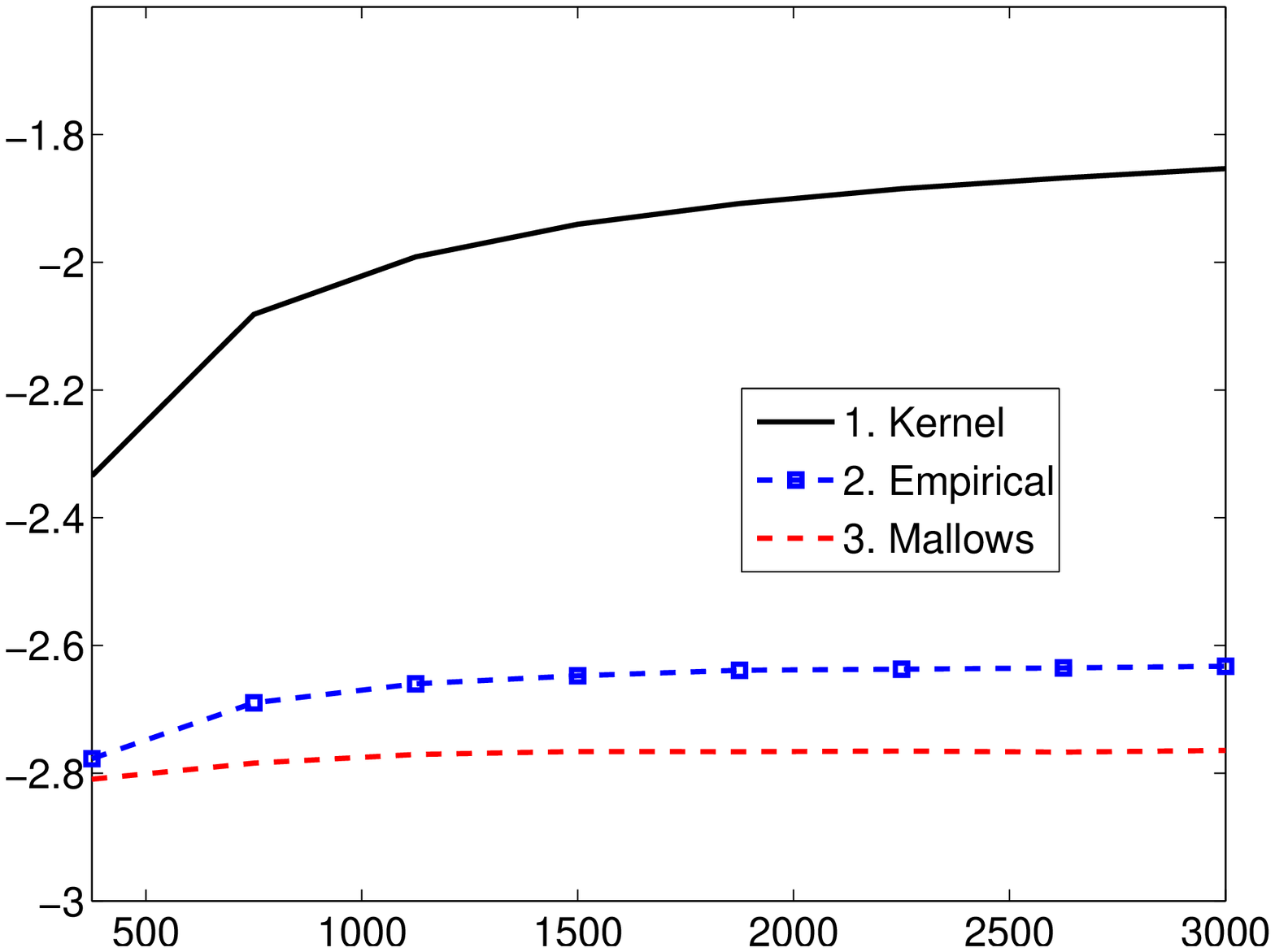} \\
\includegraphics[scale=0.31]{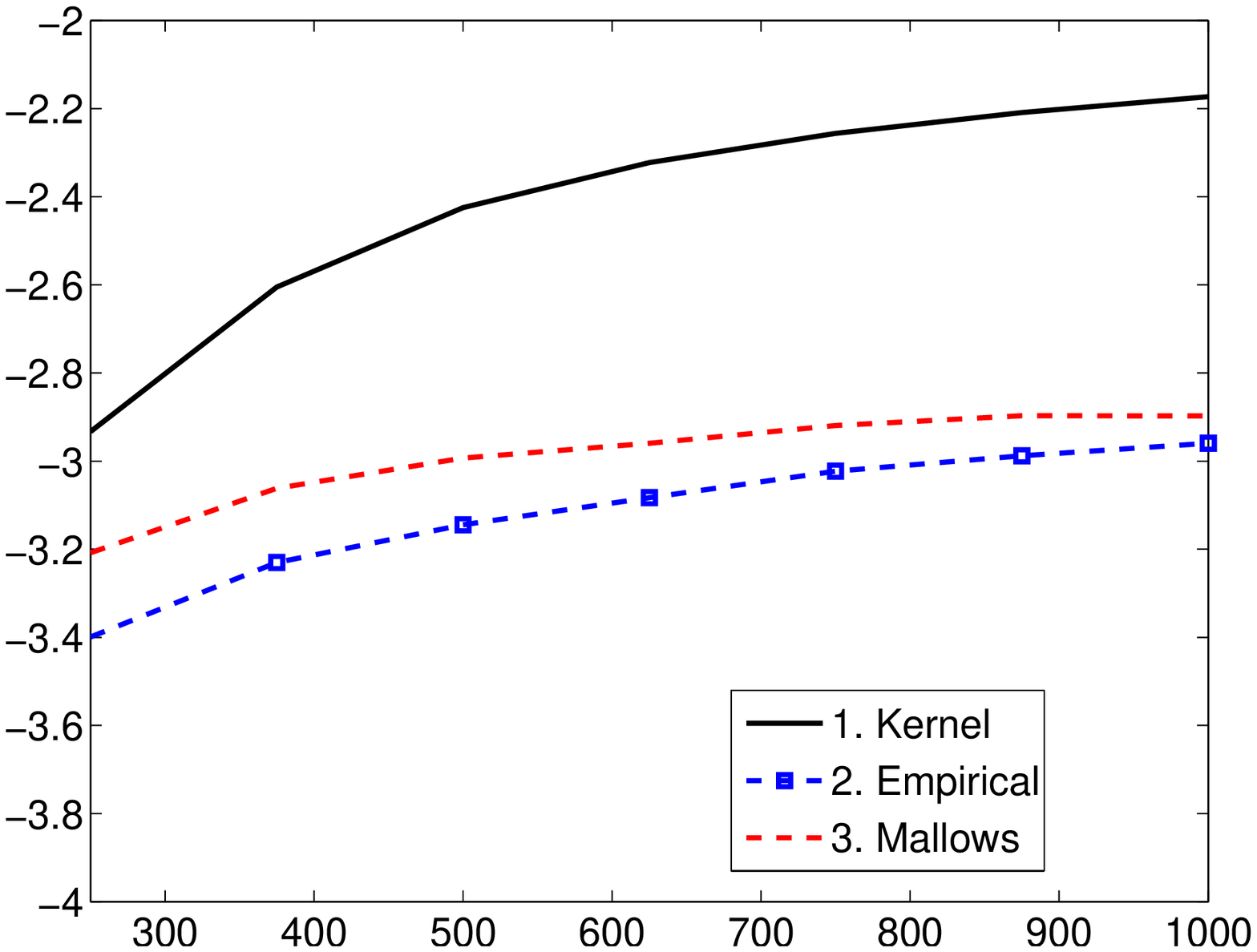}\hspace{-0.1in} &
\includegraphics[scale=0.31]{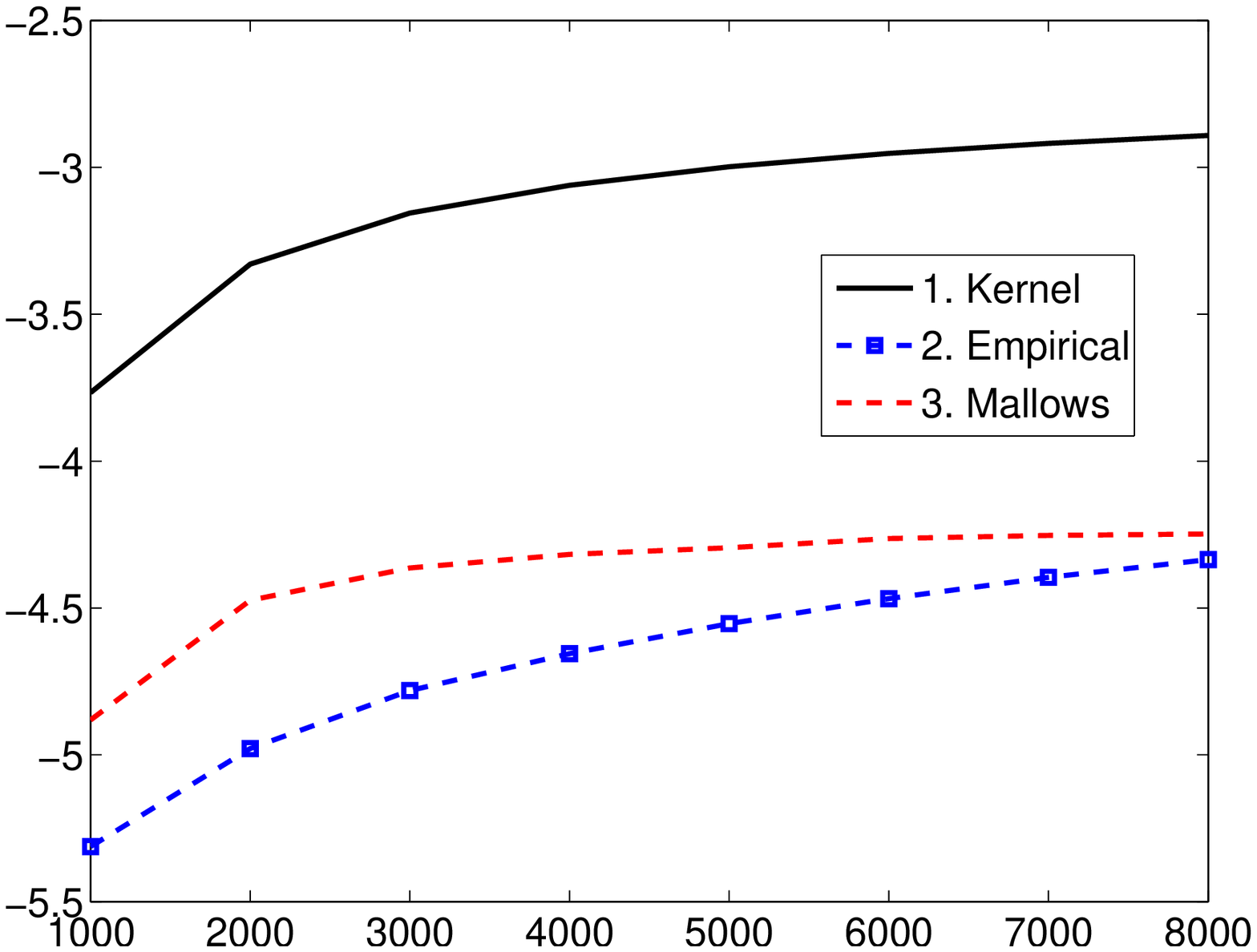}\hspace{-0.1in} &
\includegraphics[scale=0.31]{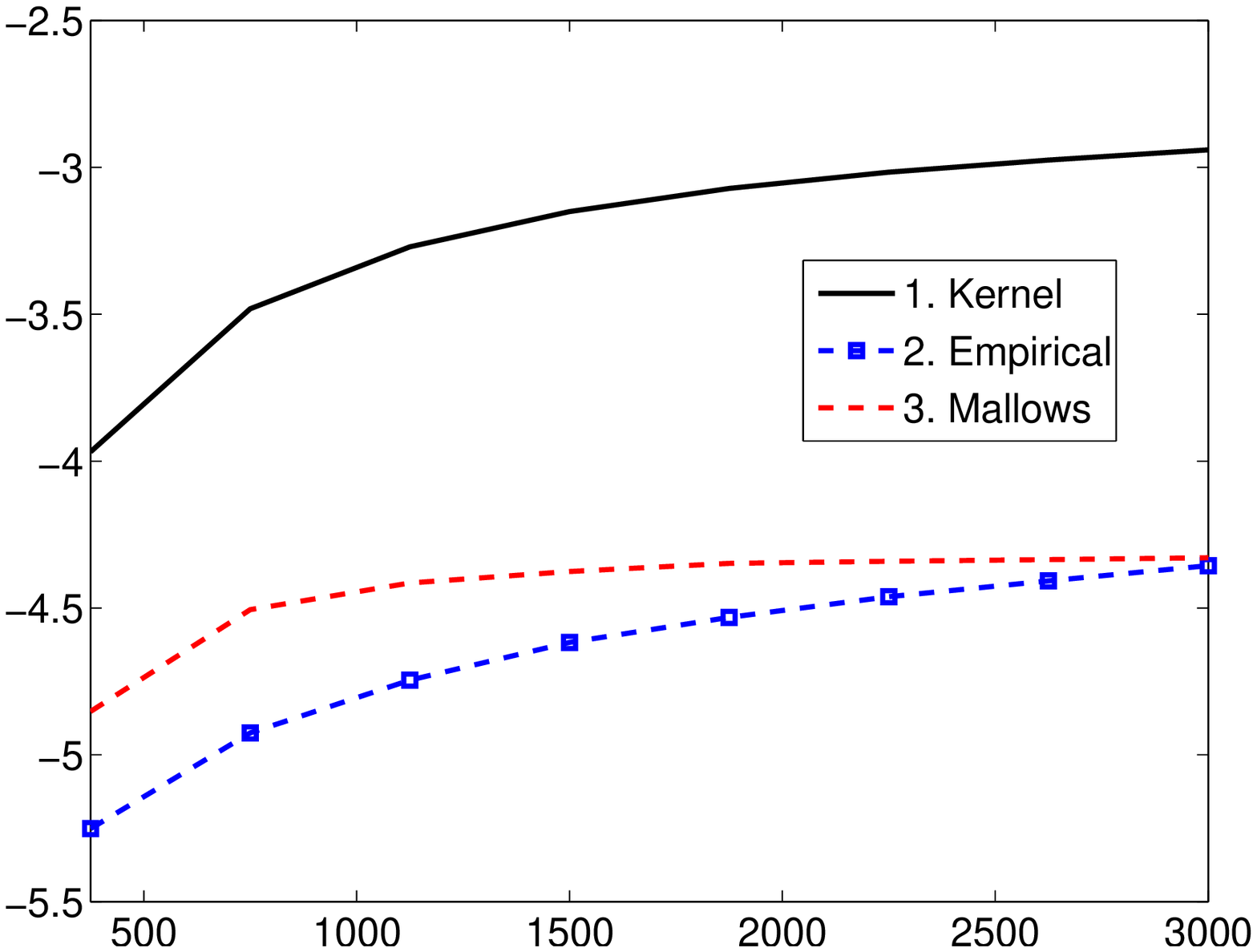} 
\end{tabular}
\caption{The test-set log-likelihood for kernel smoothing, Mallows model, and the empirical measure with respect to training size $m$ for a small number of items $n=3,4,5$ (top, middle, bottom rows) on three datasets. Both of the Mallows model (which is also intractable for large $n$ which is why $n\leq 5$ in the experiment) and the empirical measure perform worse than the kernel estimator $\hat p$.}
\label{fig:taskLoglikely}
\end{figure}

\subsection{Rank Prediction} 
Our task here is to predict ranking of new unseen items for users. We follow the standard procedure in collaborative filtering: the set of users is partitioned to two sets, a training set and a testing set. For each of the test set users we further split the observed items into two sets: one set used for estimating preferences (together with the preferences of the training set users) and the second set to evaluate the performance of the prediction \cite{Pennock2000}. Given a loss function $L(i,j)$ which measures the loss of predicting rank $i$ when true rank is $j$ (rank here refers to the number of sets of equivalent items that are more or less preferred than the current item) we evaluate a prediction rule by the expected loss.  We focus on three loss functions: $L_0(i,j)=0 \text{ if } i=j \text{ and } 1 \text{ otherwise}$, $L_1(i,j)=|i-j|$ which reduces to the standard CF evaluation technique described in \cite{Pennock2000}, and an asymmetric loss function (rows correspond to estimated number  of stars (0-5) and columns to actual number of stars (0-5) 
\begin{align}
L_e=\begin{pmatrix}
0 & 0 & 0 & 3 & 4 & 5 \\
0 & 0 & 0 & 2 & 3 & 4 \\
0 & 0 & 0 & 1 & 2 & 3 \\
9 & 4 & 1.5&0 & 0 & 0\\
12 & 6 & 3&0 & 0 & 0\\
15 & 8 & 4.5&0 & 0 & 0\\
\end{pmatrix}.
\end{align}
In contrast to the $L_0$ and $L_1$ loss, $L_e$ captures the fact that recommending bad movies as good movies is worse than recommending good movies as bad. 

For example, consider a test user whose observed preference is $3 \prec 4,5,6 \prec 10,11,12 \prec 23\prec 40,50,60\prec 100,101$. We may withhold the preferences of items $4,11$ for evaluation purposes. The recommendation systems then predict a rank of 1 for item 4 and a rank of 4 for item 11. Since the true ranking of these items are 2 and 3 the absolute value loss is $|1-2|=1$ and $|3-4|=1$ respectively. 

In our experiment, we use the kernel estimator $\hat p$ to predict ranks that minimize the posterior loss and thus adapts to customized loss functions such as $L_e$. This is an advantage of a probabilistic modeling approach over more ad-hoc rule based recommendation systems. 

Figure~\ref{fig:taskItemPredict} compares the performance of our estimator to several standard baselines in the collaborative filtering literature:  two older memory based methods vector similarity (sim1), correlation  (sim2) e.g., \cite{Breese1998}, and a recent state-of-the-art non-negative matrix (NMF) factorization (gnmf) \cite{Lawrence09}.  The kernel smoothing estimate performed similar to the state-of-the-art but substantially better than the memory based methods to which it is functionally similar.

\begin{figure}
\begin{tabular}{ccc}
{\scriptsize Movielens}& {\scriptsize Netflix} &{\scriptsize EachMovie}    \\
\includegraphics[scale=0.31]{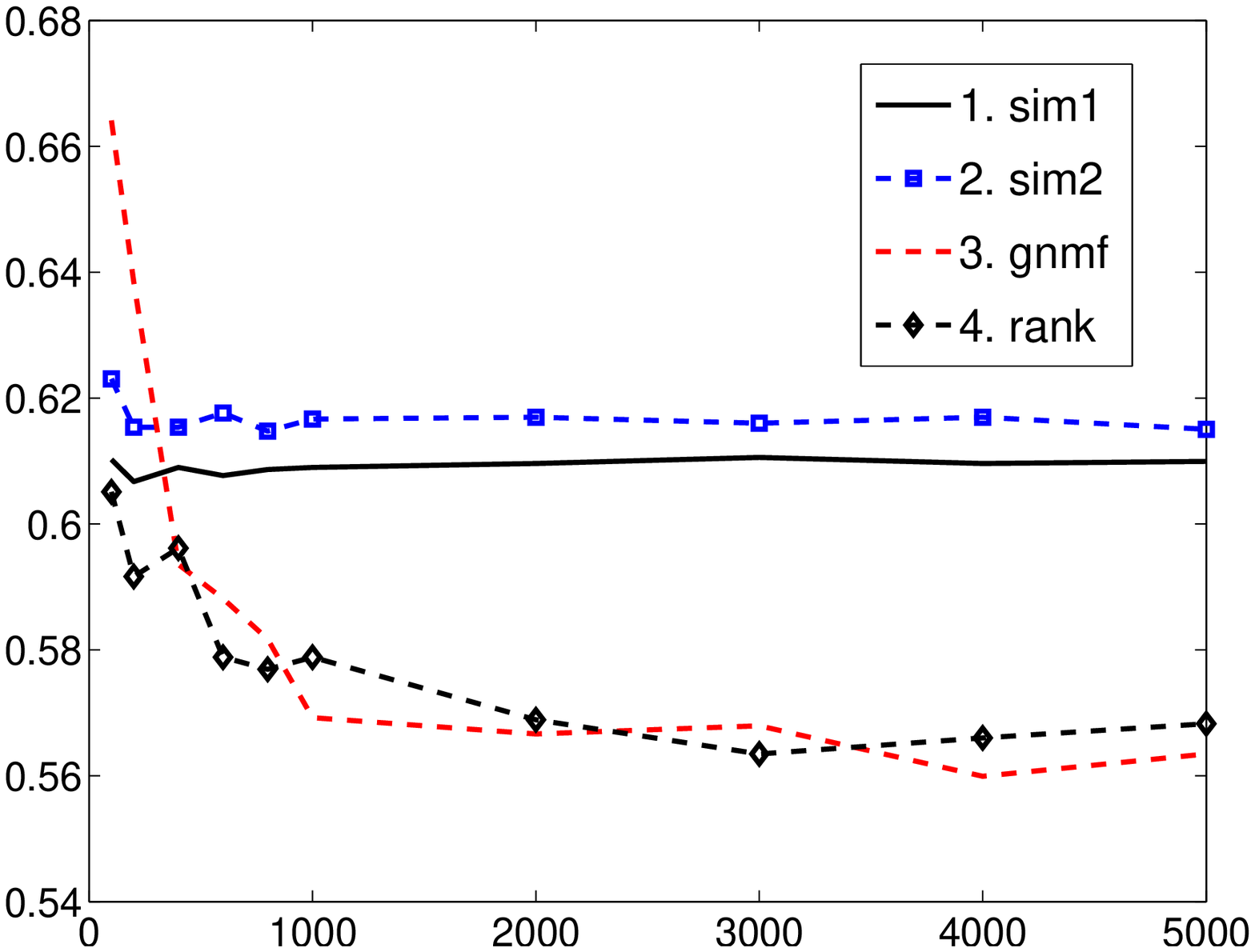}\hspace{-0.1in} &
\includegraphics[scale=0.31]{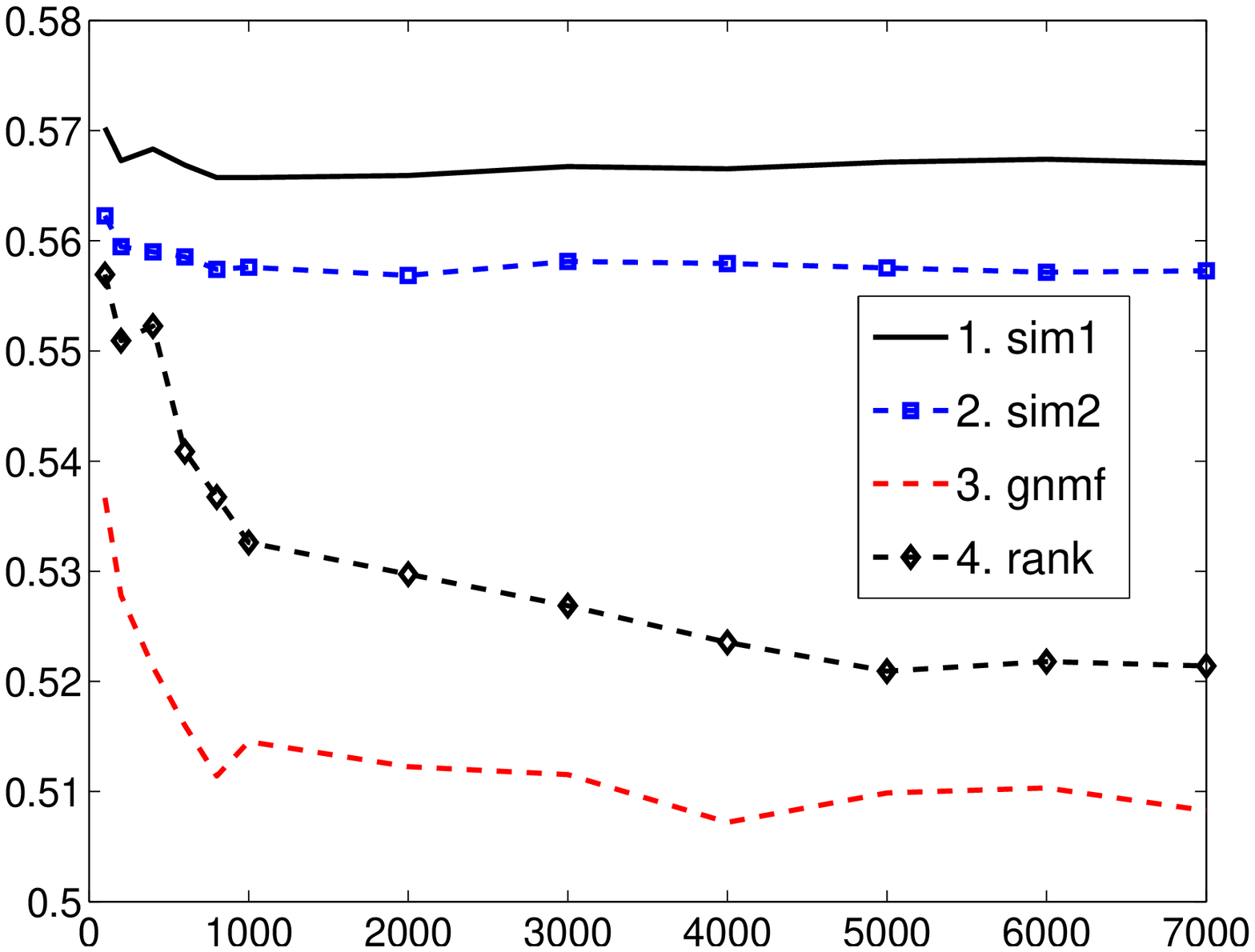}\hspace{-0.1in} &
\includegraphics[scale=0.31]{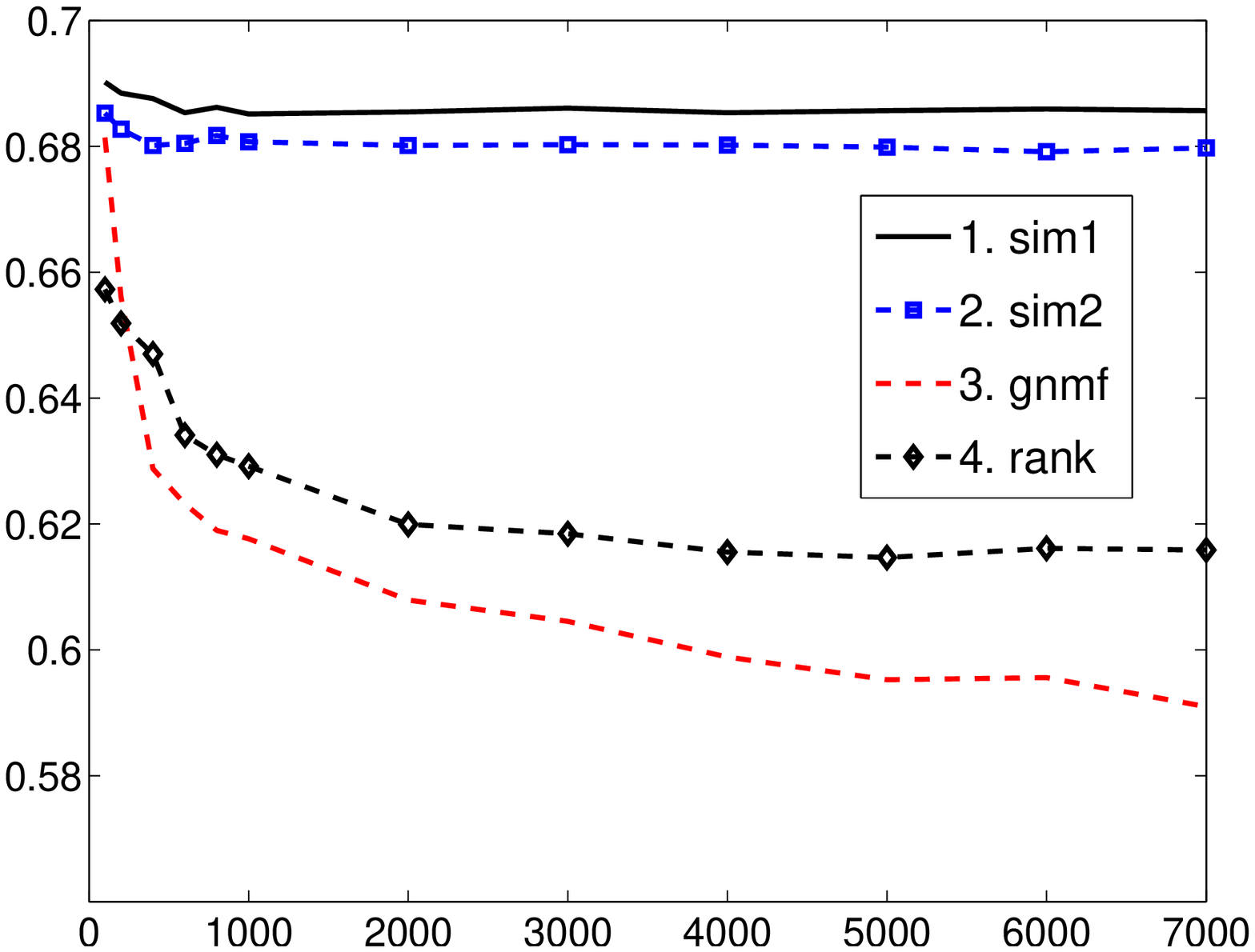} \\
\includegraphics[scale=0.31]{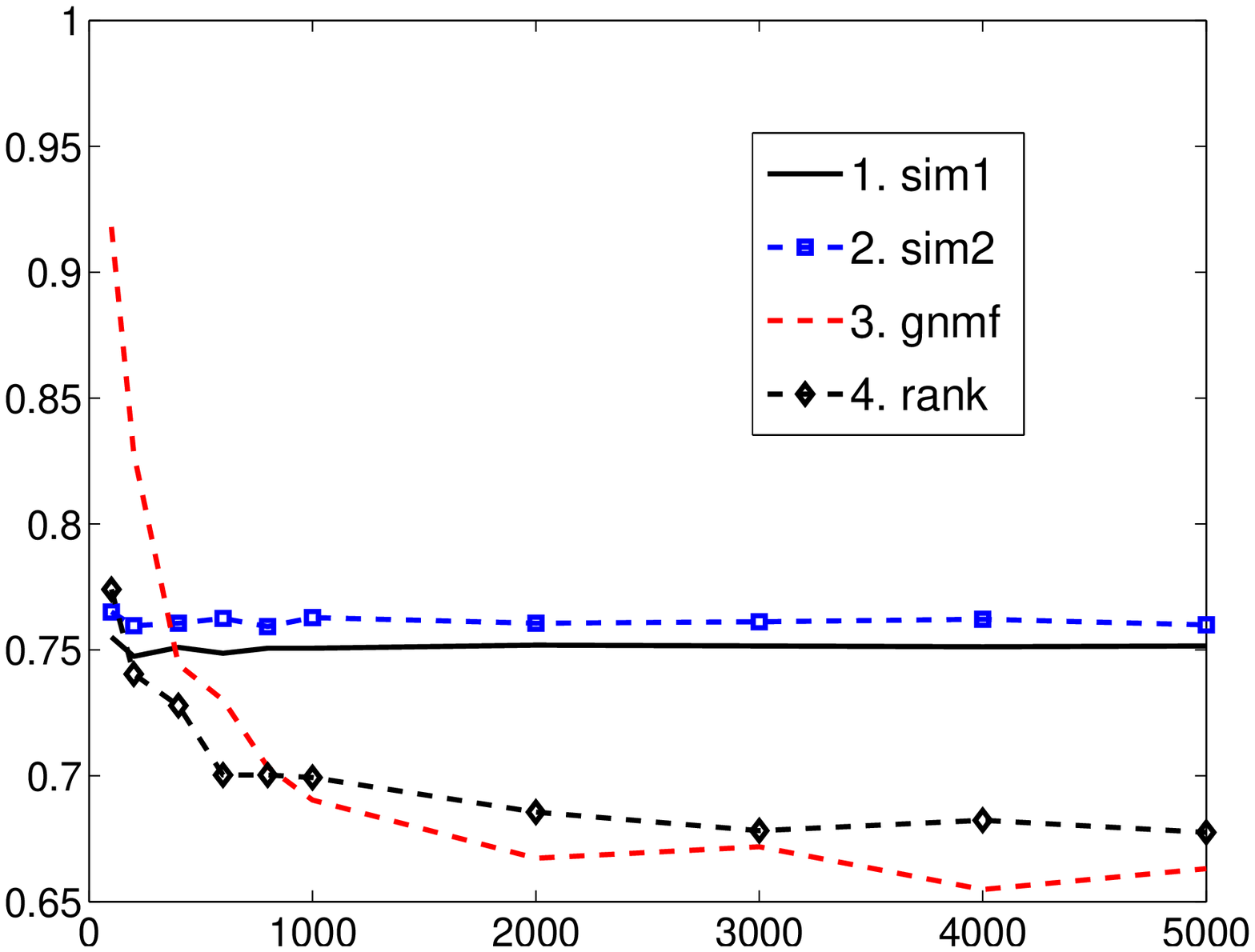}  \hspace{-0.1in} & 
\includegraphics[scale=0.31]{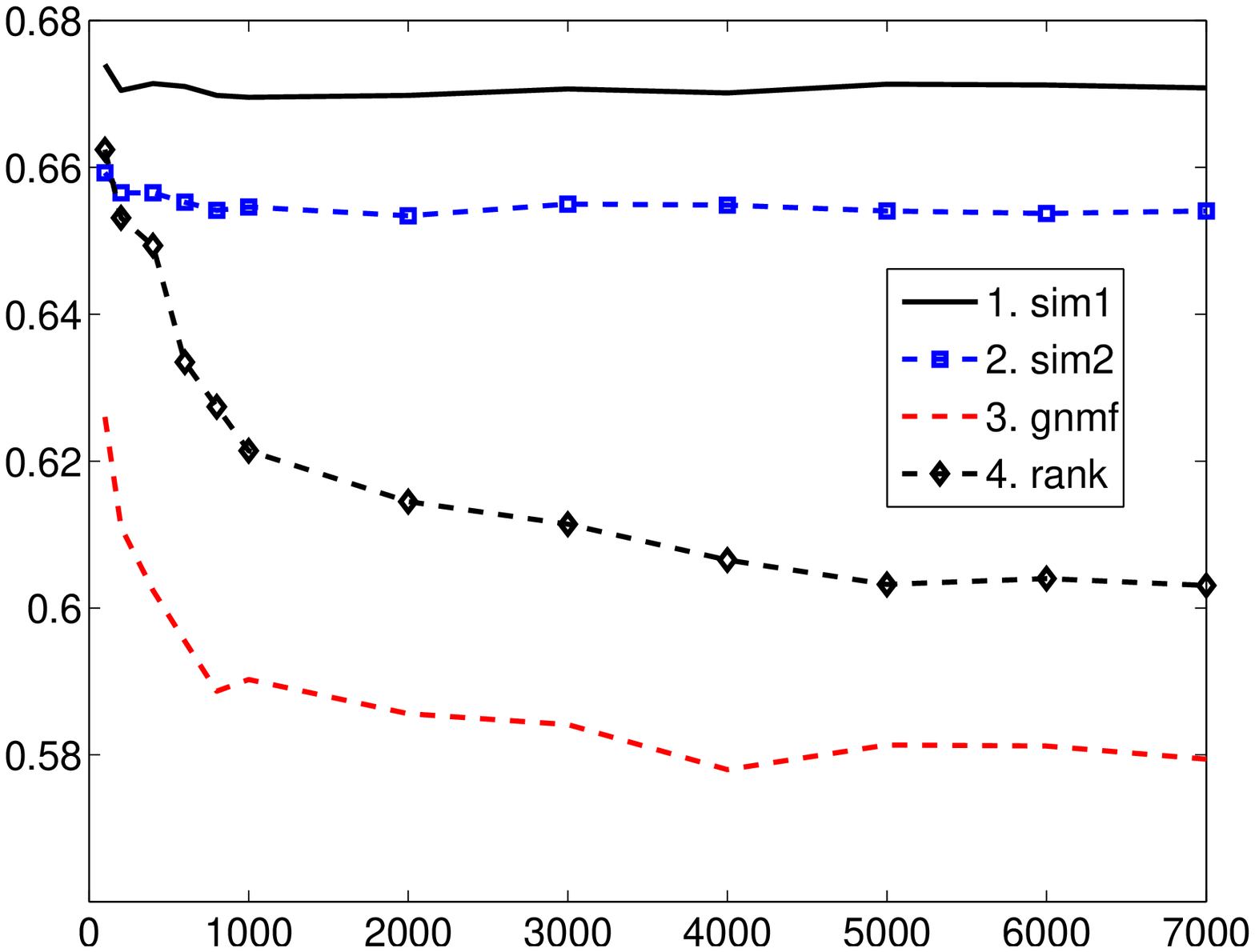}\hspace{-0.1in} &
\includegraphics[scale=0.31]{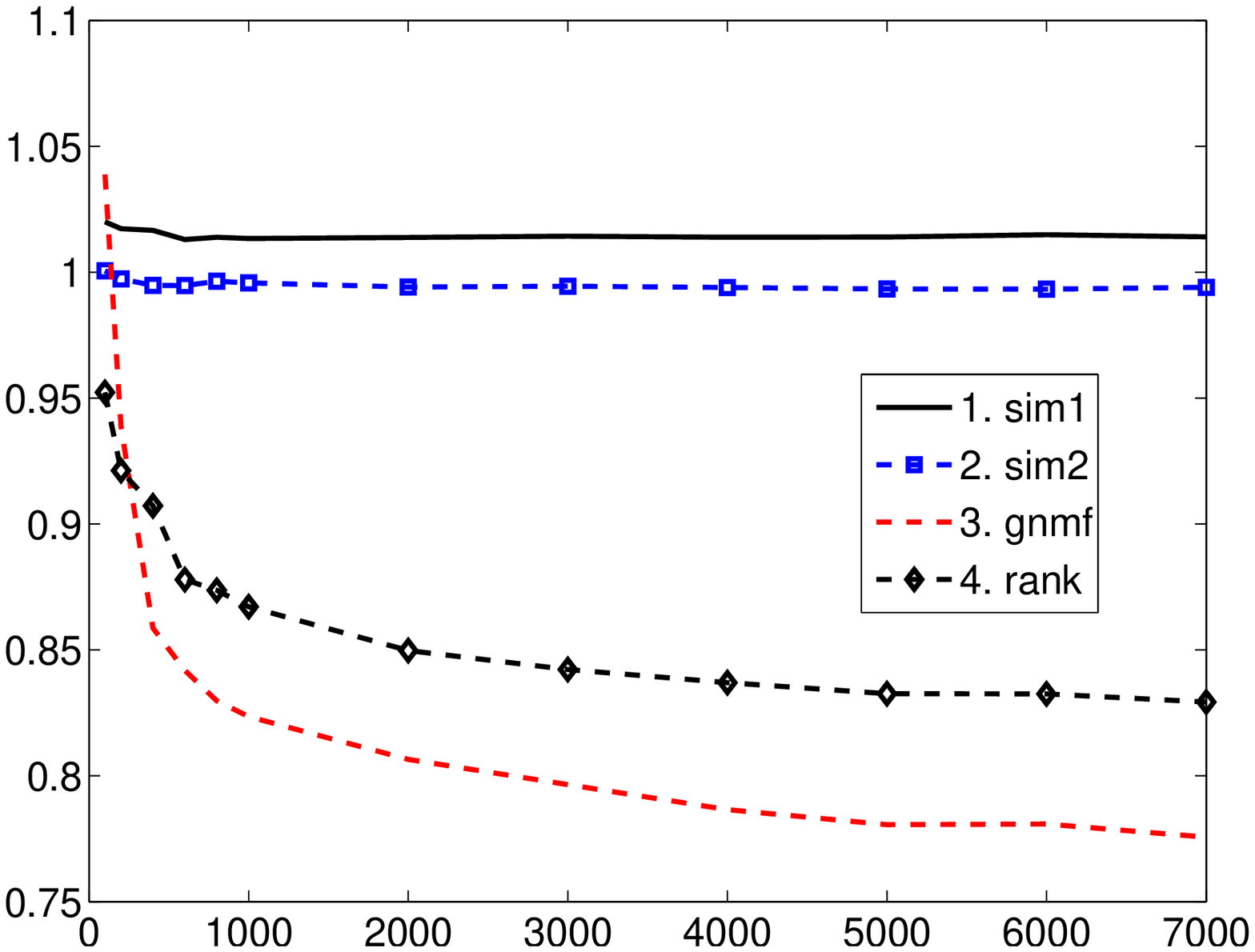}                \\
\includegraphics[scale=0.31]{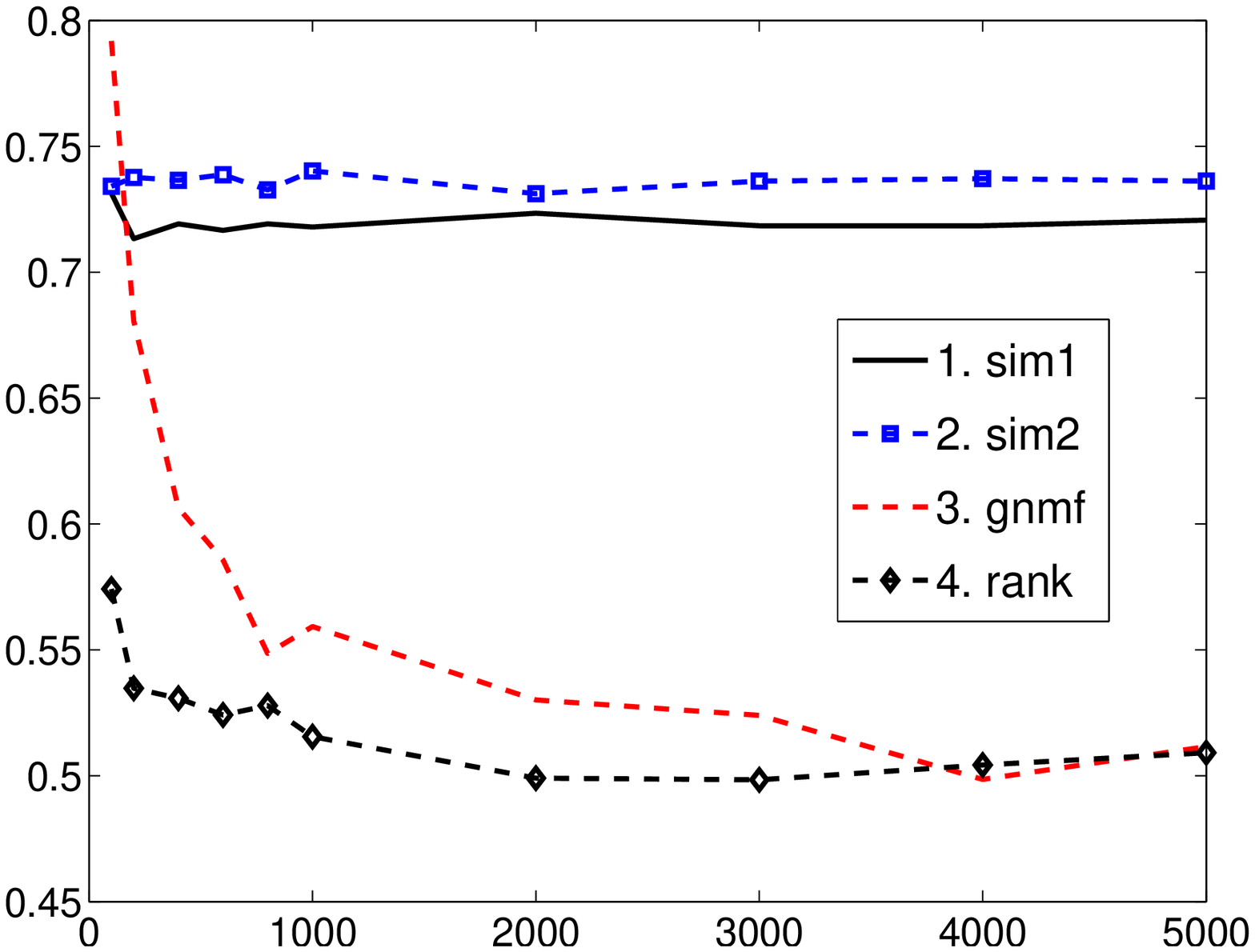}\hspace{-0.1in} &
\includegraphics[scale=0.31]{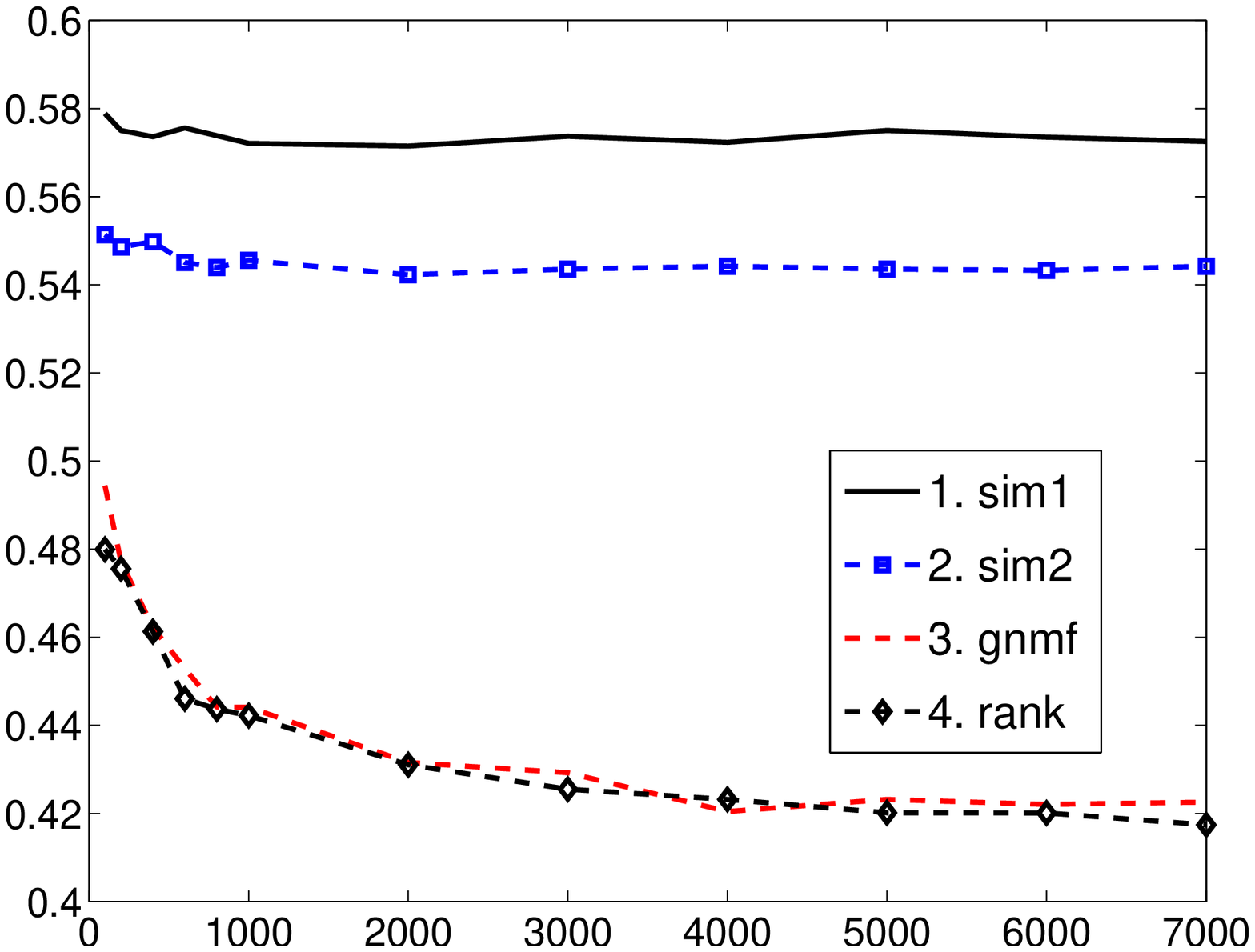}\hspace{-0.1in} &
\includegraphics[scale=0.31]{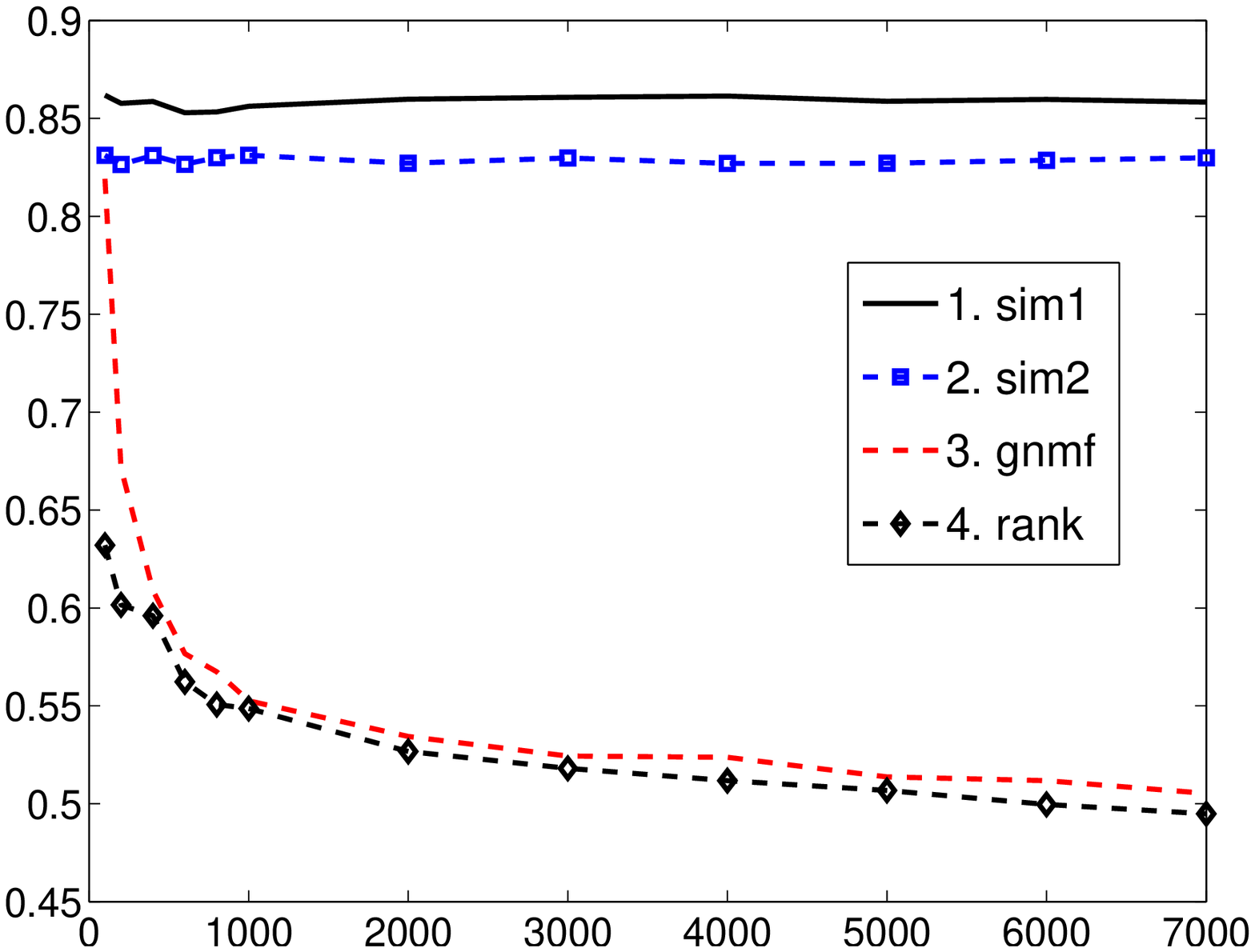} 
\end{tabular}
\caption{The prediction loss (top row: 0/1 loss $L_0$, middle row: $L_1$ loss, bottom row: asymmetric loss $L_e$) with respect to training size on three datasets. The kernel smoothing estimate performed similar to the state-of-the-art gnmf (matrix factorization) but substantially better than the memory based methods to which it is functionally similar. }
\label{fig:taskItemPredict}
\end{figure}

\subsection{Rule Discovery}
In the third task, we used the estimator $\hat p$ to detect noteworthy association rules of the type  $i\prec j \Rightarrow k\prec l$ (if $i$ is preferred to $j$ than it is probably the case that $k$ is preferred to $l$). Such association rules are important for both business analytics (devising marketing and manufacturing strategies) and recommendation system engineering. Specifically, we used $\hat p$ to select sets of four items $i,j,k,l$ for which the mutual information
$I(i\prec j\,;  k\prec l)$ is maximized. After these sets are identified we detected the precise shape of the rule (i.e., $i\prec j \Rightarrow k\prec l$ rather than $j\prec i \Rightarrow k\prec l$ by examining the summands in the mutual information expectation). 

Figure~\ref{fig:task3} (top) shows the top 10 rules that were discovered. These rules nicely isolate viewer preferences for genres such as fantasy, romantic comedies, animation, and action (note however that genre information was not used in the rule discovery). To quantitatively evaluate the rule discovery process we judge a rule $i\prec j \Rightarrow k\prec l$ to be good if $i,k$ are of the same genre and $j,l$ are of the same genre. This quantitative evaluation appears in Figure  \ref{fig:task3} (bottom) where it is contrasted with the same rule discovery process (maximizing mutual information) based on the empirical measure. 

\begin{figure} \centering
\fbox{\begin{tabular}{lll}
\footnotesize  Shrek $\prec$ 
\footnotesize  LOTR: The Fellowship of the Ring  &$\Rightarrow$& 
\footnotesize  Shrek 2$\prec$ 
\footnotesize  LOTR: The Return of the King\\
\footnotesize  Shrek $\prec$ 
\footnotesize  LOTR: The Fellowship of the Ring  &$\Rightarrow$& 
\footnotesize  Shrek 2$\prec$ 
\footnotesize  LOTR: The Two Towers\\
\footnotesize  Shrek 2 $\prec$ 
\footnotesize  LOTR: The Fellowship of the Ring  &$\Rightarrow$& 
\footnotesize  Shrek$\prec$ 
\footnotesize  LOTR: The Return of the King\\
\footnotesize  Kill Bill 2 $\prec$ 
\footnotesize  National Treasure  &$\Rightarrow$& 
\footnotesize  Kill Bill 1 $\prec$ 
\footnotesize  I. Robot\\
\footnotesize  Shrek 2 $\prec$ 
\footnotesize  LOTR: The Fellowship of the Ring  &$\Rightarrow$& 
\footnotesize  Shrek 2$\prec$ 
\footnotesize  LOTR: The Two Towers\\
\footnotesize  LOTR: The Fellowship of the Ring $\prec$ 
\footnotesize  Monsters, Inc.  &$\Rightarrow$& 
\footnotesize  LOTR: The Two Towers$\prec$ 
\footnotesize  Shrek\\
\footnotesize  National Treasure $\prec$ 
\footnotesize  Kill Bill 2  &$\Rightarrow$& 
\footnotesize  Pearl Harbor $\prec$ 
\footnotesize  Kill Bill 1\\
\footnotesize  LOTR: The Fellowship of the Ring $\prec$ 
\footnotesize  Monsters, Inc.  &$\Rightarrow$& 
\footnotesize  LOTR: The Return of the King$\prec$ 
\footnotesize  Shrek\\
\footnotesize  How to Lose a Guy in 10 Days $\prec$ 
\footnotesize  Kill Bill 2  &$\Rightarrow$& 
\footnotesize  50 First Dates$\prec$ 
\footnotesize  Kill Bill 1\\
\footnotesize  I, Robot $\prec$ 
\footnotesize  Kill Bill 2  &$\Rightarrow$& 
\footnotesize  The Day After Tomorrow $\prec$ 
\footnotesize  Kill Bill 1\\
\end{tabular}}\\
\includegraphics[scale=0.31]{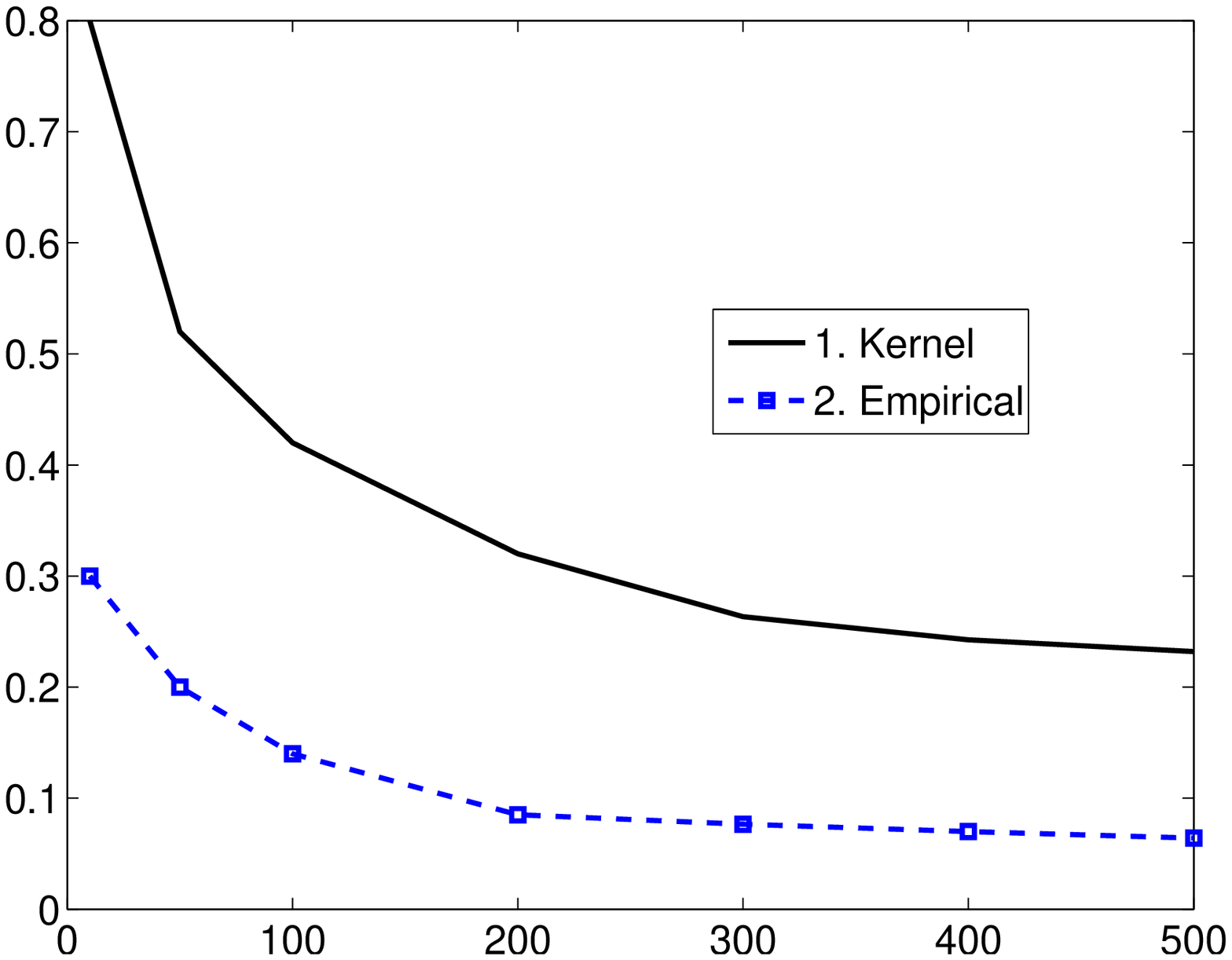}
\caption{Top: top 10 rules discovered by the kernel smoothing estimator on Netflix in terms of maximizing mutual information. Bottom: a quantitative evaluation of the rule discovery. The $x$ axis represents the number of rules discovered and the $y$ axis represents the frequency of good rules in the discovered rules. Here a rule $i\prec j \Rightarrow k\prec l$ is considered good if  $i,k$ are of the same genre and $j,l$ are of the same genre.}
\label{fig:task3}
\end{figure}

In another rule discovery experiment, we used $\hat p$ to detect association rules of the form $i \text{ ranked highest} \Rightarrow j \text{ ranked second highest}$ by selecting $i,j$ that maximize the score $\frac{p(\pi(i)=1,\pi(j)=2)}{p(\pi(i)=1)p(\pi(j)=2)}$ between pairs of movies in the Netflix data. We similarly detected rules of the form $i \text{ ranked highest} \Rightarrow j \text{ ranked lowest}$ by maximizing the scores $\frac{p(\pi(i)=1,\pi(j)=\text{last})}{p(\pi(i)=1)p(\pi(j)=\text{last})}$ between pairs of movies. 

The left panel of Figure~\ref{fig:tasks4preferPair} shows the top 9 rules of 100 most rated movies, which nicely represents movie preference of similar type, e.g. romance, comedies, and action. The right of Figure~\ref{fig:tasks4preferPair} shows the top 9 rules which represents like and dislike of different movie types, e.g. like of romance leads to dislike of action/thriller.

In a third experiment, we used $\hat p$ to construct an undirected graph where vertices are items (Netflix movies) and two nodes $i$,$j$ are connected by an edge if the average score of the rule $i \text{ ranked highest} \Rightarrow j \text{ ranked second highest}$ and the rule $j \text{ ranked highest} \Rightarrow i \text{ ranked second highest}$ is higher than a certain threshold.  Figure~\ref{fig:tasks4cluster} shows the graph for the 100 most rated movies in Netflix (only movies with vertex degree greater than 0 are shown). The clusters in the graph corresponding to vertex color and numbering were obtained using a graph partitioning algorithm and the graph is embedded in a 2-D plane using standard graph visualization technique.  Within each of the identified clusters movies are clearly similar with respect to genre, while an even finer separation can be observed when looking at specific clusters.  For example, clusters 6 and 9 both contain comedy movies, where as cluster 6 tends toward slapstick humor and cluster 9 contains romantic comedies.

\begin{figure}\centering
\begin{tabular}{ll}
\raisebox{.5in}{\fbox{\begin{tabular}{lll}
\footnotesize  Kill Bill 1 &$\Rightarrow$& \footnotesize Kill Bill 2\\
\footnotesize  Maid in Manhattan &$\Rightarrow$& \footnotesize The Wedding Planner\\
\footnotesize  Two Weeks Notice &$\Rightarrow$& \footnotesize Miss Congeniality\\
\footnotesize  The Royal Tenenbaums &$\Rightarrow$& \footnotesize Lost in Translation\\
\footnotesize  The Royal Tenenbaums &$\Rightarrow$& \footnotesize American Beauty\\
\footnotesize  The Fast and the Furious &$\Rightarrow$& \footnotesize Gone in 60 Seconds\\
\footnotesize  Spider-Man &$\Rightarrow$& \footnotesize Spider-Man 2\\
\footnotesize  Anger Management &$\Rightarrow$& \footnotesize Bruce Almighty\\
\footnotesize  Memento &$\Rightarrow$& \footnotesize Pulp Fiction\\
\end{tabular}}}
&
\raisebox{.5in}{\fbox{\begin{tabular}{lll}
\footnotesize  Maid in Manhattan &$\Rightarrow$& \footnotesize Pulp Fiction\\
\footnotesize  Maid in Manhattan &$\Rightarrow$& \footnotesize Kill Bill: 1\\
\footnotesize  How to Lose a Guy in 10 Days &$\Rightarrow$& \footnotesize Pulp Fiction\\
\footnotesize  The Royal Tenenbaums &$\Rightarrow$& \footnotesize Pearl Harbor\\
\footnotesize  The Wedding Planner &$\Rightarrow$& \footnotesize The Matrix\\
\footnotesize  Peal Harbor  &$\Rightarrow$& \footnotesize Memento\\
\footnotesize  Lost in Translation &$\Rightarrow$& \footnotesize Pearl Harbor\\
\footnotesize  The Day After Tomorrow &$\Rightarrow$& \footnotesize American Beauty\\
\footnotesize  The Wedding Planner &$\Rightarrow$& \footnotesize Raiders of the Lost Ark\\
\end{tabular}}}
\end{tabular}
\caption{Top rules discovered by kernel smoothing estimate on Netflix. Left: $\text{like A}\Rightarrow \text{like B}$. Right:  $\text{like A}\Rightarrow \text{dislike B}$.}
\label{fig:tasks4preferPair} \vspace{-.1in}
\end{figure}

\begin{figure}
\fbox{\begin{tabular}{ll}
\hline 
Cluster & Movie Titles \\
\hline  \hline
\tiny 1 & \tiny American Beauty, Lost in Translation, Pulp Fiction, Kill Bill I,II, Memento, The Royal Tenenbaums, Napoleon Dynamite,..\\ \hline 
\tiny  2 & \tiny Spider-Man, Spider-Man II \\ \hline 
\tiny 3 & \tiny Lord of the Rings I,II,III \\ \hline 
\tiny 4 & \tiny The Bourne Identity, The Bourne Supremacy\\  \hline 
\tiny 5&  \tiny Shrek, Shrek II\\ \hline 
\tiny 6 & \tiny Meet the parents, American Pie\\ \hline 
\tiny 7 & \tiny Indiana Jones and the Last Crusade, Raiders of the Lost Ark\\ \hline 
\tiny 8 & \tiny The Patriot, Pearl Harbor, Men of Honor, John Q, The General's Daughter, National Treasure, Troy, The Italian Job,..\\ \hline 
\tiny 9 & \tiny Miss Congeniality, Sweet Home Alabama,Two Weeks Notice,50 First Dates,The Wedding Planner,Maid in Manhattan,Titanic,..\\ \hline 
\tiny 10 & \tiny Men in Black I,II, Bruce Almighty, Anger Management, Mr. Deeds, Tomb Raider, The Fast and the Furious \\ \hline 
\tiny 11 & \tiny Independence Day, Con Air, Twister, Armageddon, The Rock, Lethal Weapon 4, The Fugitive, Air Force One\\
 \hline
\end{tabular}}
\includegraphics[scale=1]{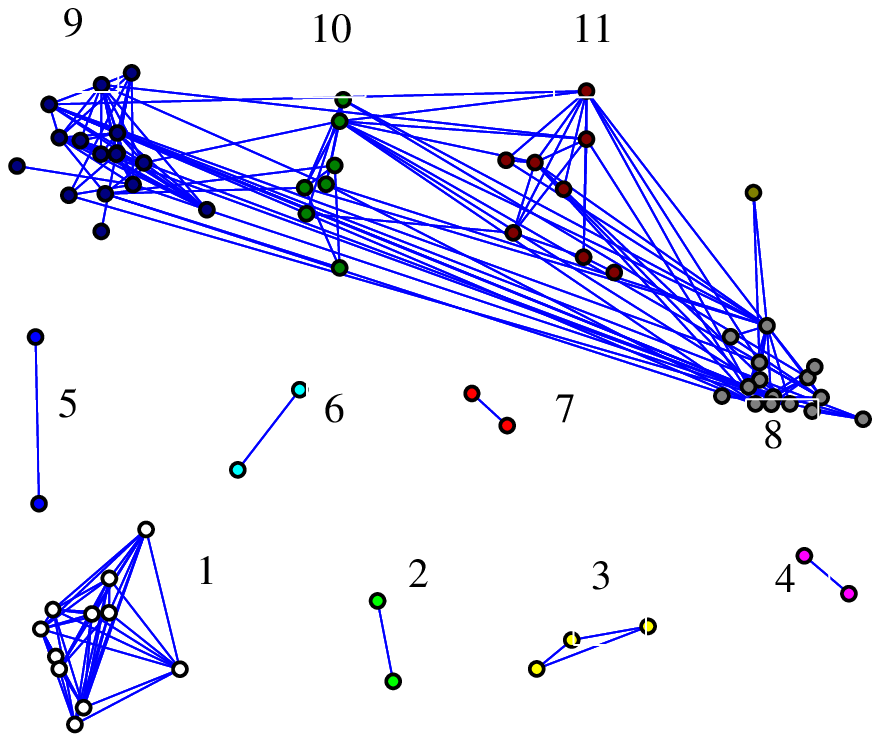}
\caption{A graph corresponding to the 100 most rated Netflix movies where edges represent high affinity as determined by the rule discovery process (see text for more details).}
\label{fig:tasks4cluster}
\end{figure}

\section{Related Work}
Collaborative filtering or recommendation system has been an active research area in computer science since the 1990s. The earliest efforts made a prediction for the rating of items based on the similarity of the test user and the training users  \cite{Resnick1994,Breese1998,Herlocker1999}. Specifically, these attempts used similarity measures such as Pearson correlation \cite{Resnick1994} and Vector cosine similarity \cite{Breese1998,Herlocker1999} to evaluate the similarity level between different users. 

More recent work includes user and movie clustering 
\cite{Breese1998,Ungar1998,Xue2005}, item-item similarities \cite{Sarwar2001}, Bayesian networks \cite{Breese1998}, dependence network \cite{Heckerman2000} and probabilistic latent variable models \cite{Pennock2000,Hofmann2004,marlin2004}.

Most recently, the state of the art methods including the winner of the Netflix competition are based on non-negative matrix factorization of the partially observed user-rating matrix. The factorized matrix can be used to fill out the unobserved entries in a way similar to latent factor analysis \cite{Jester,Rennie2005,Lawrence09,Koren2010}.

Each of the above methods focuses exclusively on user ratings. In some cases item information is available (movie genre, actors, directors, etc) which have lead to several approaches that combine voting information with item information e.g.,  \cite{Basu1998,Popescul2001,Schein2002}.

Our method differs from the methods above in that it constructs a full probabilistic model on preferences, it is able to handle heterogeneous preference information (not all users must specify the same number of preference classes) and does not make any parametric assumptions. In contrast to previous approaches it enables not only the prediction of item ratings, but also the discovery of association rules and the estimation of probabilities of interesting events.

\section{Summary}  
Estimating distributions from tied and incomplete data is a central task in many applications with perhaps the most obvious one being collaborative filtering. An accurate estimator $\hat p$ enables going beyond the traditional item-rank prediction task. It can be used to compute probabilities of interest, find association rules, and perform a wide range of additional data analysis tasks. 

We demonstrate the first non-parametric estimator for such data that is computationally tractable i.e., polynomial rather than exponential in $n$. The computation is made possible using generating function and dynamic programming techniques. 

We examine the behavior of the estimator $\hat p$ in three sets of experiments. The first set of experiments involves estimating probabilities of interest such as $p(i\prec j)$. The second set of experiments involves predicting preferences of held-out items which is directly applicable in recommendation systems. In this task, our estimator outperforms other memory based methods (to which it is similar functionally)  and performs similarly to state-of-the-art methods that are based on non-negative matrix factorization. In the third set of experiments we examined the usage of the estimator in discovering association rules such as $i\prec j \Rightarrow k\prec l$.  

\bibliographystyle{plain}     {\small \bibliography{../../common/externalPapers,../../common/groupPapers}}

\begin{thebibliography}{10}

\bibitem{Basu1998}
C.~Basu, H.~Hirsh, and W.~Cohen.
\newblock Recommendation as classification: Using social and content-based
  information in recommendation.
\newblock In {\em Proceedings of the National Conference on Artificial
  Intelligence}, pages 714--720, 1998.

\bibitem{Breese1998}
J.~Breese, D.~Heckerman, and C.~Kadie.
\newblock Empirical analysis of predictive algorithms for collaborative
  filtering.
\newblock In {\em Proc. of Uncertainty in Artificial Intelligence}, 1998.

\bibitem{Diaconis:88}
P.~Diaconis.
\newblock {\em Group Representations in Probability and Statistics}.
\newblock Institute of Mathematical Statistics, 1988.

\bibitem{Jester}
K.~Goldberg, T.~Roeder, D.~Gupta, and C.~Perkins.
\newblock Eigentaste: A constant time collaborative filtering algorithm.
\newblock {\em Information Retrieval}, 4(2):133--151, 2001.

\bibitem{Heckerman2000}
David Heckerman, David~Maxwell Chickering, Christopher Meek, Robert
  Rounthwaite, and Carl Kadie.
\newblock Dependency networks for inference, collaborative filtering, and data
  visualization.
\newblock {\em Journal of Machine Learning Research}, 1, 2000.

\bibitem{Herlocker1999}
J.L. Herlocker, J.A. Konstan, A.~Borchers, and J.~Riedl.
\newblock An algorithmic framework for performing collaborative filtering.
\newblock In {\em Proceedings of the 22nd annual international ACM SIGIR
  conference on Research and development in information retrieval}, page 237.
  ACM, 1999.

\bibitem{Hofmann2004}
T.~Hofmann.
\newblock Latent semantic models for collaborative filtering.
\newblock {\em ACM Transactions on Information Systems (TOIS)}, 22(1):115,
  2004.

\bibitem{Koren2010}
Y.~Koren.
\newblock Factor in the neighbors: Scalable and accurate collaborative
  filtering.
\newblock {\em ACM Transactions on Knowledge Discovery from Data (TKDD)},
  4(1):1--24, 2010.

\bibitem{Lawrence09}
N.D. Lawrence and R.~Urtasun.
\newblock Non-linear matrix factorization with gaussian processes.
\newblock In {\em Proceedings of the 26th Annual International Conference on
  Machine Learning}, pages 601--608. ACM, 2009.

\bibitem{Lebanon2008}
G.~Lebanon and Y.~Mao.
\newblock Non-parametric modeling of partially ranked data.
\newblock {\em Journal of Machine Learning Research}, 9:2401--2429, 2008.

\bibitem{Mallows1957}
C.~L. Mallows.
\newblock Non-null ranking models.
\newblock {\em Biometrika}, 44:114--130, 1957.

\bibitem{Marden1996}
J.~I. Marden.
\newblock {\em Analyzing and modeling rank data}.
\newblock CRC Press, 1996.

\bibitem{marlin2004}
B.~Marlin.
\newblock Modeling user rating profiles for collaborative filtering.
\newblock In {\em Advances in Neural Information Processing Systems}, 2004.

\bibitem{Pennock2000}
D.~M. Pennock, E.~Horvitz, S.~Lawrence, and C.~L. Giles.
\newblock Collaborative filtering by personality diagnosis: A hybrid memory-
  and model-based approach.
\newblock In {\em Proceedings of the 16th Conference on Uncertainty in
  Artificial Intelligence}, 2000.

\bibitem{Popescul2001}
A.~Popescul, L.~H. Ungar, D.~M. Pennock, and S.~Lawrence.
\newblock Probabilistic models for unified collaborative and content-based
  recommendation in sparse-data environments.
\newblock In {\em Proceedings of the 17th Conference on Uncertainty in
  Artificial Intelligence}, 2001.

\bibitem{Rennie2005}
J.D.M. Rennie and N.~Srebro.
\newblock Fast maximum margin matrix factorization for collaborative
  prediction.
\newblock In {\em Proceedings of the 22nd international conference on Machine
  learning}, page 719. ACM, 2005.

\bibitem{Resnick1994}
P.~Resnick, N.~Iacovou, M.~Suchak, P.~Bergstrom, and J.~Riedl.
\newblock Grouplens: an open architecture for collaborative filtering of
  netnews.
\newblock In {\em Proceedings of the Conference on Computer Supported
  Cooperative Work}, 1994.

\bibitem{Sarwar2001}
B.~Sarwar, G.~Karypis, J.~Konstan, and J.~Reidl.
\newblock Item-based collaborative filtering recommendation algorithms.
\newblock In {\em Proceedings of the 10th international conference on World
  Wide Web}, 2001.

\bibitem{Schein2002}
A.I. Schein, A.~Popescul, L.H. Ungar, and D.M. Pennock.
\newblock Methods and metrics for cold-start recommendations.
\newblock In {\em Proceedings of the 25th annual international ACM SIGIR
  conference on Research and development in information retrieval}, pages
  253--260. ACM, 2002.

\bibitem{Stanley2000}
R.~P. Stanley.
\newblock {\em Enumerative Combinatorics}, volume~1.
\newblock Cambridge University Press, 2000.

\bibitem{Ungar1998}
L.H. Ungar and D.P. Foster.
\newblock Clustering methods for collaborative filtering.
\newblock In {\em AAAI Workshop on Recommendation Systems}, 1998.

\bibitem{Xue2005}
G.R. Xue, C.~Lin, Q.~Yang, W.S. Xi, H.J. Zeng, Y.~Yu, and Z.~Chen.
\newblock Scalable collaborative filtering using cluster-based smoothing.
\newblock In {\em Proceedings of the 28th annual international ACM SIGIR
  conference on Research and development in information retrieval}, pages
  114--121. ACM New York, NY, USA, 2005.

\end{thebibliography}
\end{document}